%% file: main.tex
\pdfoutput=1
\PassOptionsToPackage{hyphens}{url}
\documentclass[11pt]{arvix}
\usepackage[authoryear]{natbib}

\input{math_commands.tex}

\usepackage{microtype}
\usepackage{graphicx}
\usepackage{subcaption}
\usepackage{wrapfig}
\usepackage{float}
\usepackage{flafter}
\usepackage{placeins}
\newfloat{algorithm}{tbp}{loa}
\floatname{algorithm}{Algorithm}
\usepackage{booktabs}
\usepackage{array}
\usepackage{enumitem}
\newcommand{\compacttable}{\renewcommand{\arraystretch}{0.92}%
  \setlength{\aboverulesep}{0.35ex}\setlength{\belowrulesep}{0.45ex}}
\usepackage{amsthm}
\newtheorem{theorem}{Theorem}
\renewcommand{\eqref}[1]{\textup{(\ref{#1})}}
\usepackage{multirow}
\PassOptionsToPackage{hyphens}{url}
\usepackage{hyperref}
\usepackage{url}
\usepackage{xcolor}
\input{styles/prompt_boxes.tex}

\title{From Token-Max to Outcome-Max: How You Use AI Determines Its Productivity}
\author[1]{Chen Xu}
\author[2]{Mengqiao Liu\textsuperscript{*}}
\author[1]{Beibei Li}
\author[1]{Chenyan Xiong}
\affil[1]{Carnegie Mellon University}
\affil[2]{Amazon}
\correspondingauthor{\texttt{\{chenxu2,beibeili,cx\}@andrew.cmu.edu}, \texttt{mengqiao.liu7@gmail.com}}
\paperurl={https://github.com/chenxu2-gif/TokenEcom}
\hypersetup{pdfauthor={Chen Xu, Mengqiao Liu, Beibei Li and Chenyan Xiong},
  pdftitle={From Token-Max to Outcome-Max: How You Use AI Determines Its Productivity}}

\newcommand{\MethodName}{OutcomeShare}

\begin{document}

\begin{abstract}
Generative artificial intelligence (AI) models can perform increasingly complex tasks, yet greater AI usage does not necessarily translate into proportional productivity gains. We identify token-max as one source of this inefficiency: when token consumption is treated as productive effort, agents are encouraged to over-exert and expend computation beyond what is necessary. We instead propose \emph{outcome-max}, which rewards independently verified task completion per unit cost and induces a principled stopping rule.
Then, to study these objectives, we develop a three-level simulation framework spanning immediate interaction, long-run behavioral adaptation, and organizational collaboration. Across all three levels, outcome-max improves the efficiency of AI-assisted production while largely preserving verified task performance. To further align these incentives with outcome-max, we introduce \emph{\MethodName}, an incentive mechanism. Theory and simulation show that \MethodName{} can induce participation while generating shared gains for employees, firms, and LLM providers. Together, our results suggest that AI productivity not just depends on model capability, but also on the objectives governing AI use.
\end{abstract}
\maketitle
\begingroup
\renewcommand{\thefootnote}{\fnsymbol{footnote}}
\footnotetext[1]{This work is unrelated to Mengqiao Liu's position at Amazon.}
\endgroup

\input{sections/01_introduction}
\input{sections/02_related}
\input{sections/03_setup}
\input{sections/04_mechanism}
\input{sections/05_experiments}
\input{sections/07_conclusion}
\label{end:body}
\clearpage

\input{sections/08_statements}

\label{end:maintext}

\bibliography{refs}
\bibliographystyle{plainnat}

\clearpage
\appendix
\raggedbottom
\renewcommand{\topfraction}{0.92}
\renewcommand{\bottomfraction}{0.85}
\renewcommand{\textfraction}{0.08}
\renewcommand{\floatpagefraction}{0.75}
\setcounter{topnumber}{3}
\setcounter{bottomnumber}{3}
\setcounter{totalnumber}{5}
\makeatletter
\setlength{\@fptop}{0pt}
\setlength{\@fpsep}{12pt}
\setlength{\@fpbot}{0pt plus 1fil}
\makeatother
\renewcommand{\thefigure}{S\arabic{figure}}
\renewcommand{\thetable}{S\arabic{table}}
\renewcommand{\theHfigure}{supp.\arabic{figure}}
\renewcommand{\theHtable}{supp.\arabic{table}}
\setcounter{figure}{0}
\setcounter{table}{0}

\input{appendix/A_protocols}
\FloatBarrier
\input{appendix/B_theory}

\FloatBarrier
\input{appendix/C_results}

\input{appendix/D_prompts}

\end{document}

%% file: math_commands.tex
\usepackage{amsmath,amsfonts,bm}

\def\eqref#1{equation~\ref{#1}}

\def\1{\bm{1}}

\DeclareMathAlphabet{\mathsfit}{\encodingdefault}{\sfdefault}{m}{sl}
\SetMathAlphabet{\mathsfit}{bold}{\encodingdefault}{\sfdefault}{bx}{n}



%% file: styles/prompt_boxes.tex
\usepackage{framed}
\definecolor{PromptBorder}{HTML}{B9C8D6}
\definecolor{PromptBackground}{HTML}{F5F8FA}
\definecolor{PromptTitle}{HTML}{24445D}
\newenvironment{promptbox}[1]{%
  \setlength{\OuterFrameSep}{7pt}%
  \MakeFramed{\advance\hsize-\width\FrameRestore}%
  \noindent{\sffamily\bfseries\fontsize{10}{12}\selectfont
    \color{PromptTitle}#1}\par\nobreak\vspace{5pt}\nobreak
  \fontsize{10}{12}\selectfont\sffamily\raggedright
  \setlength{\parindent}{0pt}\setlength{\parskip}{3pt}%
  \emergencystretch=2em\relax
}{\par\endMakeFramed}

%% file: sections/01_introduction.tex
\section{Introduction}\label{sec:intro}

\begin{wrapfigure}{r}{2.1in}
\centering
\vspace{-8pt}
\includegraphics[width=\linewidth]{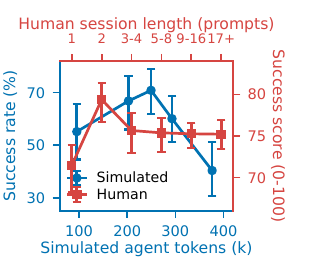}
\caption{\textbf{More tokens need not improve success.} Blue: simulated token-max on SWE-bench. Red: real
human coding sessions from SWE-chat.}
\label{fig:motivation}
\end{wrapfigure}

Generative artificial intelligence (AI) can improve productivity in software development~\citep{cui2025genai} and professional writing~\citep{noyzhang2023}. Yet realizing these gains requires balancing greater adoption with its costs. New reports of token-usage leaderboards~\citep{nyt_tokenmaxxing2026} and Meta's plans to limit employee token spending~\citep{meta_usage2026} reveal a tension between encouraging use and controlling costs: \textit{why can greater AI usage fail to improve productivity?}

We argue that one source of this inefficiency is the \emph{token-max} policy, which treats greater token consumption as evidence of productive effort. Yet more tokens do not necessarily produce better outcomes. Figure~\ref{fig:motivation} illustrates this pattern in both our simulations and real developer sessions: success initially increases but then plateaus or declines as interactions grow longer. The reason is due to \emph{over-efforting} effect: excessive interaction expands the reasoning context, encourages unnecessary and redundant checking and revision.


These observations motivate an alternative objective. Inspired by economic theory for rewarding inputs versus rewarding outcomes~\cite{lazear2000performance}, we propose \emph{outcome-max}, which rewards independently verified task completion per unit cost. Under this objective, further computation is justified when its expected marginal return exceeds the return achieved so far.

To evaluate the effectiveness of our proposed objective, we develop a three-level simulation framework spanning single-task execution, textual-memory adaptation, and employee selection. We simulate token-max and outcome-max behavior using LLM-based agents explicitly instructed to either maximize token
use or achieve correct results at minimum cost. The results show that the outcome-max policy substantially reduces token costs while even slightly improving overall verified task performance.

To further incentivize outcome-max behavior, we design \emph{\MethodName}, a mechanism that combines verified-outcome payments with liability for defective delivery. Intuitively, it shares the value of verified outcomes to reward employees for avoiding wasteful effort, while imposing penalties for
defective delivery to discourage them from misrepresenting unsuccessful work.
Through theory and scenario accounts, we examine
such a mechanism can align provider, employee, and firm incentives while
creating shared gains. This motivates us not only to develop more capable AI models, but also to design mechanisms that guide how people use them effectively.

We make three contributions as follows. Code is available at
\url{https://github.com/chenxu2-gif/TokenEcom}.
\begin{itemize}[leftmargin=*]
\item We identify the inefficiencies caused by rewarding token consumption and propose \emph{outcome-max}, an alternative policy that optimizes verified production per unit cost.
\item We develop a three-level simulation framework and demonstrate that outcome-max substantially reduces token costs while even slightly improving overall verified task performance.
\item We introduce \MethodName, an incentive mechanism that aligns organizational stakeholders around the adoption of outcome-max.
\end{itemize}

%% file: sections/02_related.tex
\section{Related Work}\label{sec:related}

In this section, we aim to connect agent computation, economic incentives,
and behavioral simulation as foundations for productive AI use.

\paragraph{Agent computation and adaptation.}
ReAct combines reasoning and actions~\citep{yao2023react}; MetaGPT organizes collaboration~\citep{hong2024metagpt}.
Executable benchmarks assess correctness~\citep{jimenez2024,liu2023evalplus}.
Scaling and cost-aware evaluation study accuracy--compute tradeoffs~\citep{snell2024,kapoor2024}.
Overthinking can waste computation and impair agent
performance~\citep{chen2024overthinking,cuadron2025overthinking}.
BAGEN studies budget awareness and early stopping~\citep{lin2026bagen}.
Reflexion adapts verbal memory~\citep{shinn2023}; ACE evolves reusable
playbooks~\citep{zhang2026ace}.
GEPA evolves prompts using execution feedback~\citep{agrawal2026gepa}.
We evaluate these decisions under specified supervision policies.

\paragraph{Incentives and productivity.}
Principal--agent theory links measurement to effort~\citep{holmstrom1979,holmstrom1991}.
Imperfect performance measures can distort incentives~\citep{baker1992},
including in multi-agent systems~\citep{rauba2026principal}.
Recent work studies contracts for model and effort choice~\citep{sreenivas2026contracting}
and usage versus success-based pricing~\citep{lobel2026pricing}.
We combine outcome payments, defect liability, and three-party participation.
Workplace studies find productivity gains~\citep{brynjolfsson2025,cui2025genai}.
Effects vary across settings~\citep{vaccaro2024}; early-2025 tools can also slow
developers~\citep{becker2025productivity}.

\paragraph{Behavioral simulation.}
Social simulation includes generative agents and Concordia~\citep{park2023,vezhnevets2023concordia}.
CAMEL and SOTOPIA study role-based interaction~\citep{li2023camel,zhou2024sotopia}.
\citet{gao2024abm} survey agent-based simulation.
Simulated-subject studies compare human responses~\citep{aher2023turing,argyle2023out}.
Economic simulations study decisions~\citep{horton2023homo} and market
institutions, including liability~\citep{erlei2026markets}.
We evaluate costs and verified quality conditional on specified behavior;
incentive-induced behavior requires a separate test.

%% file: sections/03_setup.tex
\section{Framework: From Token-Max to Outcome-Max}
\label{sec:setup}

In this section, we aim to formalize AI supervision around interaction
volume or verified output per cost, then simulate the corresponding supervisory behaviors across
execution, adaptation, and selection (Figure~\ref{fig:framework}).

\subsection{Notation and observable outcomes}
We distinguish externally verified success $y$, employee acceptance $a$,
and agent token use $T$. Success and acceptance are binary; a shipped defect
is $d=a(1-y)$. The external evaluator scores the final artifact even when
the employee does not accept it. Employees are persona-conditioned LLMs,
not human participants (Appendices~\ref{app:setup} and~\ref{app:validity}).

Providers meter tokens in four categories
$r\in\{\mathrm{in},\mathrm{out},\mathrm{cr},\mathrm{cw}\}$: \emph{input}
tokens the agent reads (task text, files, tool results), \emph{output} tokens
it writes (reasoning, commands, patches), \emph{cache reads} (input served
from a prompt cache at a discount), and \emph{cache writes} (tokens newly
stored to that cache). With $T^r$ the token count in category $r$ and $p_r$
its price per token, compute cost and the reported token quantity $T$ are
\begin{equation}
C^{\mathrm{AI}}=\sum_r p_rT^r,\qquad T=T^{\mathrm{in}}+T^{\mathrm{out}}.
\label{eq:cost}
\end{equation}
The cost includes all four categories; the reported token metric counts
uncached input and output. Supervision cost uses an assumed wage and time
proxy, so dollar payoffs depend on the scenario in Table~\ref{tab:pnl}.
These quantities separate verified output, resource use, and financial cost.

\subsection{Token-max: optimize the resource metric}
Let $\pi$ denote the interaction policy under the common resource limits.
Token-max rewards expected token consumption over the task distribution:
\begin{equation}
J_{\mathrm{tok}}(\pi)=\mathbb E_{\pi}[T].
\label{eq:token-objective}
\end{equation}
Additional tokens improve this objective even without a better verified
result. This objective property does not predict behavior. Our simulation
implements token-max through requests for continued work; token billing
motivates the archetype without establishing that billing alone produces it.

\subsection{Outcome-max: optimize verified output per cost}
Outcome-max rewards verified task completion per unit of compute cost:
\begin{equation}
J_{\mathrm{out}}(\pi)=\frac{\mathbb E_{\pi}[y]}{\mathbb E_{\pi}[C^{\mathrm{AI}}]},
\label{eq:outcome-objective}
\end{equation}
where expected cost is positive. Additional computation is worthwhile only
when its expected resolution gain justifies its cost
(Appendix~\ref{app:theory}). The ratio uses workload totals, not the mean
of task-level efficiencies. Success and cost together determine whether more computation is productive.

\subsection{Simulation framework}
\label{sec:simulation-framework}
\begin{wrapfigure}{r}{3.0in}
\centering
\vspace{-14pt}
\includegraphics[width=\linewidth]{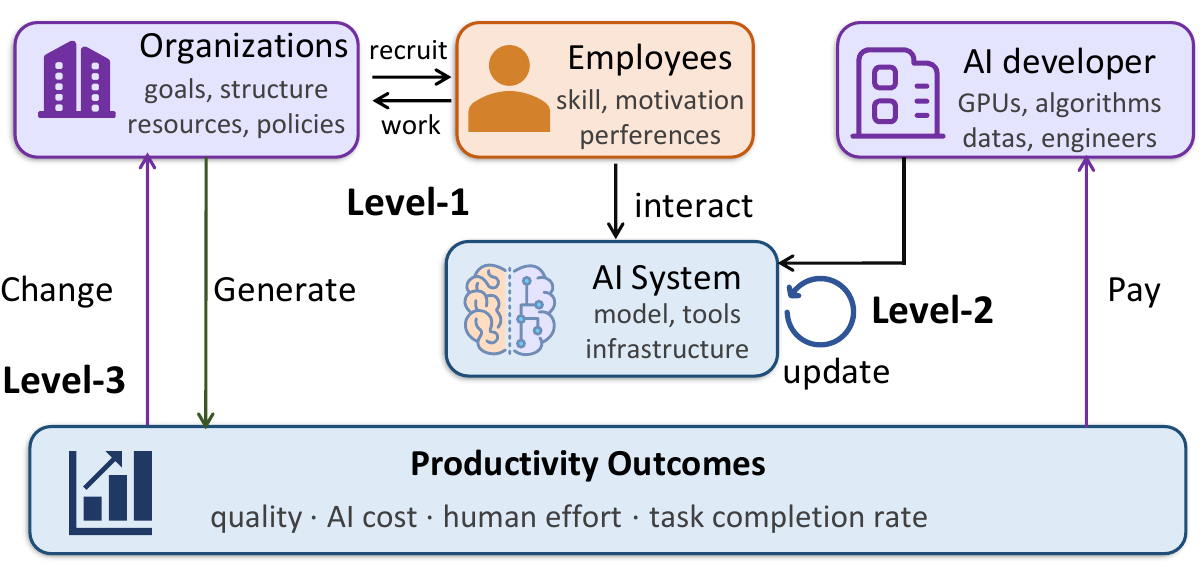}
\caption{\textbf{Simulation framework.} L1: execution; L2: adaptation;
L3: selection.}
\label{fig:framework}
\vspace{-10pt}
\end{wrapfigure}
We simulate token-max and outcome-max as two prescribed supervisory
behaviors (Figure~\ref{fig:framework}). Instructions specify how the employee
continues or stops; the task, models, evaluator, and resource limits are
held fixed within each paired comparison.
Employee personas vary supervisory behavior; neither the employee nor the
production agent sees the external gold outcome during an interaction.

\noindent\textbf{L1: execution.} Employee $i$ with objective $o$ supervises
the agent on task $x_j$: after each exchange $t$ the employee accepts
($a_t=1$) or asks for another step, until acceptance or a resource limit,
$\tau=\min\{t\le\tau_{\max}:a_t=1\}$. External evaluation checks the final
artifact, $y_{ij}=E(x_j,H_{ij})$. L1 evaluates policy consequences without
learning or workforce turnover.

\noindent\textbf{L2: adaptation.} Employee and agent retain separate textual
tactic libraries $L_g$ and $K_g$. After each training generation $g$,
reflection sees only outcome feedback $F_g=\{(y,a,d,\tau,T)_{ij}\}$ and updates
$L_{g+1}=R_H(L_g,F_g;o)$ and $K_{g+1}=R_A(K_g,F_g)$, where employee reflection $R_H$ revises tactics within the specified
behavioral regime and agent reflection $R_A$ targets verified resolution
at lower cost. Feedback
excludes task text, artifacts, and transcripts. Ablations freeze either
library; joint adaptation updates both. Model weights stay fixed, and frozen
libraries are tested on held-out tasks (Appendix~\ref{app:methods}).

\noindent\textbf{L3: selection.} Each period, the firm ranks employees by
average token use or verified solves per compute cost, with a bounded
preference for smaller successful patches. It retains top-ranked employees
and recruits their mutated descendants (Appendix~\ref{app:setup}). This tests
workforce composition and performance under the imposed retention rule.
Together, L1--L3 evaluate policy consequences under specified update rules.

%% file: sections/04_mechanism.tex
\section{Mechanism: \MethodName}
\label{sec:pricing}

In this section, we aim to specify conditions for aligning provider,
employee, and firm incentives through \MethodName: payments for verified work, liability for defects, and
shared value.

\subsection{From an objective to an implementable contract}
The intuition behind \MethodName{} is to share the value of verified outcomes
among the provider, employee, and firm, thereby rewarding productive work
rather than unnecessary token consumption. At the same time, liability for
defective delivery discourages employees from approving or misrepresenting
unsuccessful work as successful. The outcome-sharing and liability terms are
fixed before production, and participation constraints ensure that each
stakeholder benefits under the specified valuation and cost assumptions.

\subsection{Mechanism \MethodName}
Let $V$ be value per verified output, $Q(\phi)$ weekly verified
output under fee share $\phi\in[0,1]$, $p>0$ the effective audit probability,
$F$ the fine charged to an employee who approves a defect, and $\kappa V>0$
the damage per defective delivery. The three payment quantities are
\begin{equation}
R_{\mathrm{dev}}(\phi)=\phi V Q(\phi),\qquad
\lambda=\frac{pF}{\kappa V},\qquad pF=\kappa V\ \ (\lambda=1).
\label{eq:contract}
\end{equation}
Provider revenue is before compute costs; an employee bonus shares verified
value. Normalized liability $\lambda$ distinguishes no liability ($\lambda=0$)
from damage-matched liability ($\lambda=1$), our alignment target; experiments
also test other levels. Defects incur penalties without outcome payments.

\paragraph{Participation and fee cap.}
To preserve participation, compare each payoff $U_k(\phi,pF)$ with baseline
$U_k^0$ under common valuation and cost assumptions:
\begin{equation}
U_k(\phi,\kappa V)\ge U_k^0,\qquad
k\in\{\mathrm{provider},\mathrm{firm},\mathrm{employee}\}.
\label{eq:feasible-fee}
\end{equation}
The fee cap is the highest share satisfying all three constraints;
fees only divide existing surplus.

\subsection{Theoretical analysis}
The results below characterize objective properties and verification losses
(Appendix~\ref{app:theory}), not LLM optimization or contract adoption.

\begin{theorem}[Objective response to extra computation]
\label{prop:continuation}
Let $\pi,\pi'$ be feasible policies on the same task distribution with
positive expected cost. If $\pi'$ increases expected tokens and cost while
leaving positive expected verified output unchanged, then
$J_{\mathrm{tok}}(\pi')>J_{\mathrm{tok}}(\pi)$ and
$J_{\mathrm{out}}(\pi')<J_{\mathrm{out}}(\pi)$.
More generally, for baseline output $Y$, cost $C>0$, and changes
$\Delta Y$ and $\Delta C>0$, outcome-max improves if and only if
$\Delta Y/\Delta C>Y/C$.
\end{theorem}

Matching expected defect costs aligns private and social verification effort.

\begin{theorem}[Pigouvian verification alignment]
\label{prop:pigou}
Let verification effort $e\in\mathcal E$ cost $c(e)$ and induce shipped-defect
probability $\delta(e)$, with external damage $\kappa V$ per defect.
Assume a risk-neutral employee, a fixed effective audit probability $p>0$,
a collectible fine $F$, and all other private and social payoff terms constant
in $e$. If $\mathcal E$ is nonempty and both losses attain their minima, then
$pF=\kappa V$ implies that private loss $c(e)+pF\delta(e)$ and social loss
$c(e)+\kappa V\delta(e)$ have identical minimizers.
\end{theorem}

\paragraph{Implications.} Theorem~\ref{prop:continuation} favors continuation
when its marginal verified gain per dollar exceeds the current average.
Hard-task losses therefore motivate better verification signals, rather than
uniformly shorter work. Theorem~\ref{prop:pigou} aligns verification by matching
expected liability to defect damage.

\subsection{From efficiency to shared gains}
Efficient production creates surplus; fees and bonuses divide it. Shared gains
require value sufficient to cover production, supervision, verification, and
defects while satisfying participation constraints (Eq.~\eqref{eq:feasible-fee}).
Thus efficiency creates the opportunity, while quality control and allocation
determine whether participants benefit.

%% file: sections/05_experiments.tex
\section{Experiments}
\label{sec:experiments}

In this section, we aim to evaluate the consequences of token-max and
outcome-max supervisory behaviors across L1--L3. RQ1--RQ3 compare their performance
and sensitivity; RQ4--RQ5 evaluate delivery rules and scenario payoffs.

\subsection{Experimental setting}
\label{sec:experimental-setting}

\paragraph{Tasks and external evaluation.}
We evaluate externally verifiable work across software engineering,
program synthesis, and mathematics. SWE-bench Verified requires repairing
repository issues, with correctness determined by containerized gold
tests~\citep{jimenez2024,openai_swebench_verified2024}. HumanEval+ evaluates
generated programs with extended tests~\citep{liu2023evalplus}. GSM8K and
MATH-500 evaluate extracted final answers through normalized
comparison~\citep{cobbe2021,lightman2023}.
SWE-bench supports execution, adaptation, and workforce-selection experiments;
the other benchmarks test whether the execution-level contrast extends beyond
repository repair. Neither the employee nor the production agent sees the
gold evaluation outcome during an interaction. An employee's decision to
accept therefore remains distinct from whether the submitted artifact is
correct.

Difficulty groups are defined from benchmark metadata or reference-solution
structure, rather than from the observed success of either objective.
Software repair uses estimated human repair time, competition mathematics
uses benchmark difficulty labels, and the remaining tasks use
reference-solution complexity proxies. These proxies describe task structure
and need not rank model success monotonically. The mathematics stratification
also relies on reconstructed task ordering. Appendix~\ref{app:difficulty}
provides the grouping rules and their limitations.

\paragraph{Models, personas, and the objective contrast.}
The production agent performs the task, while a separate employee model
requests work, reviews the response, and decides whether to continue.
Employee personas vary how clearly they describe the task and how actively
they scrutinize the result. Profiles cover broad guidance and checking styles
as well as individual behaviors such as asking for examples, clarifying goals,
and requesting revisions. The behavioral dimensions and their prompt encoding
are given in Table~\ref{tab:personas} and Appendix~\ref{app:prompts}.

Within each paired comparison, the task, persona, model assignment,
evaluator, and resource limits are held fixed. The main treatment changes
the supervisory instructions: token-max encourages continued interaction,
whereas outcome-max emphasizes a correct result with economical resource use.
This compares specified policy packages; it does not establish that a payment
objective alone causes those behaviors to emerge. Objective-only prompts
provide a separate, narrower check of that distinction. Model assignments
also differ between the static, adaptation, and cross-task cohorts, so
absolute success differences across these cohorts do not isolate learning
or model capability.

\paragraph{Execution, adaptation, and selection protocols.}
The execution experiment crosses a common task roster with fixed employee
profiles under both objectives. Interaction ends when the employee accepts
or a shared resource limit is reached; final artifacts are evaluated
externally regardless of acceptance. This separates the consequences of
supervision within a task from changes in policy across tasks.

The adaptation experiment adds reusable employee and agent playbooks.
After training, reflection receives outcome summaries and the current
playbook, but not task text, artifacts, or transcripts. It updates the
employee playbook, the agent playbook, or both, depending on the ablation.
Employee reflection retains the assigned supervisory objective, while agent
reflection targets successful production with lower resource use.
Model weights remain fixed. Updated playbooks are frozen before evaluation
on disjoint held-out tasks, and held-out outcomes are never passed back to
reflection. Thus the comparison tests transfer of learned routines rather
than reuse of test-task solutions; Algorithm~\ref{alg:bilateral} gives the
update procedure.

The selection experiment instead maintains a workforce with diverse starting
profiles. Each period, employees supervise tasks from a common backlog,
receive a retention score, and update their tactics. The firm replaces
lower-ranked employees with mutated descendants of higher-ranked employees.
Token-based retention rewards interaction volume; outcome-based retention
rewards verified work per compute cost, with a bounded preference for smaller
successful patches. Recruitment preserves workforce size and records
lineage, making it possible to relate changes in supervisory style to
production and firm accounts. This is selection under an imposed rule,
rather than employees voluntarily choosing a new objective.
Appendix~\ref{app:setup} retains task counts, model configurations,
interaction limits, and numerical settings for each protocol.

\paragraph{Metrics.}

For $N$ attempted tasks, $y_j$ is external correctness and $T_j$ is the
production agent's input-plus-output tokens. We report three complementary
metrics, with efficiency measured in verified solutions per million agent tokens:
\begin{equation}
\mathrm{Success}=100\frac{\sum_j y_j}{N},\qquad
\overline T=\frac{\sum_j T_j}{N},\qquad
\eta=10^6\frac{\sum_j y_j}{\sum_j T_j}.
\label{eq:eta}
\end{equation}
External success is distinct from employee acceptance.
\textit{Efficiency improvement} is $100(\eta_{\mathrm{out}}/\eta_{\mathrm{tok}}-1)$.
This resource metric is distinct from monetary cost, which depends on model
pricing and supervision.
\emph{Economic outcomes} comprise provider revenue, employee income, and firm
net profit; Appendix~\ref{app:methods} defines their scenario accounting.

\paragraph{Matched comparisons and uncertainty.}
The main performance table summarizes valid-run cohorts. Supporting analyses
restrict comparisons to shared task--persona--repetition cells, so differences
in which attempts have usable evaluations do not determine the comparison.
Task-cluster bootstrap intervals keep repeated attempts on the same task
together. They describe uncertainty over the observed task mix, conditional
on the learned policies and settings, rather than variability across new
learning histories or independently simulated firms. Numerical boldface
indicates a better point estimate, not statistical significance.

\paragraph{Human interaction comparisons.}
WildChat coding conversations provide a check of interaction-depth coverage
(Appendix~\ref{app:validity}); SWE-chat provides observational comparisons of
session length, user personas, and rated success~\citep{zhao2024wildchat,baumann2026swechat}.
These datasets are not assigned token-max or outcome-max incentives, and
human session ratings are not executable correctness tests. Their role is
to contextualize the simulated behaviors, rather than to validate a causal
claim about workplace productivity.

\subsection{Comparing token-max and outcome-max}
\label{sec:objective-experiments}

\subsubsection{RQ1: Does outcome-max outperform token-max?}
\label{sec:rq1}
\label{sec:overeffort}
\label{sec:overeffort-cross}
\label{sec:entrench}

In this section, we aim to compare verified work per agent token under
the two specified supervisory behaviors across the three levels
(Table~\ref{tab:rq1}). Difficulty groups are defined independently of observed
success; paired intervals appear in Appendix~\ref{app:paired-uncertainty}.

\input{tables/tab_rq1.tex}

At L1, outcome-max reduces token use across task families while aggregate
success remains broadly comparable or improves. The resource-use gap is
therefore more consistent than the quality gap. Hard software and
program-synthesis tasks show the important exception: stopping earlier can
remove useful revision as well as redundant work. The aggregate efficiency
advantage should consequently be read together with difficulty-specific
success, rather than as evidence that shorter interaction is always preferable.

At L2, the efficiency advantage persists whether the employee, agent, or
both adapt. Bilateral adaptation improves success across difficulty groups,
whereas agent-only adaptation trades some success for lower resource use.
This distinction is consistent with the different roles of production and
supervision: a more economical production routine does not itself ensure
that the employee recognizes when further checking is needed. Matched
comparisons retain the efficiency pattern, while making the agent-only
success ranking less certain. The results support adapting supervision
alongside production when quality matters, without establishing a universal
advantage across task distributions.

At L3, outcome-based selection preserves more verified value while lowering
production and supervision costs, reversing the firm's average net sign in
the stated accounting scenario. Token-based selection can instead retain
styles that generate extensive activity without comparable verified output.
The organizational consequence is that an interaction-volume metric affects
not only the current task, but also which supervisory practices persist
through recruitment and retention.

The neutral reference uses zero behavior propensities, no incentive directive,
and no employee follow-up; it is unpaired with the adaptation arms and should
not be interpreted as an isolated estimate of the effect of adaptation.
Table~\ref{tab:matched-l2} provides the matched adaptation comparison, while
the stopping diagnostics below examine why efficiency gains can coexist
with hard-task failures.

\subsubsection{RQ2: Is the objective gap robust?}
\label{sec:robustness}
\label{sec:rq2}
In this section, we aim to assess policy sensitivity across backbones,
personas, and prompt designs (Figure~\ref{fig:robustness}). Positive
success gaps and efficiency gains favor outcome-max.

\begin{figure}[!htb]
\centering
\includegraphics[width=0.95\linewidth]{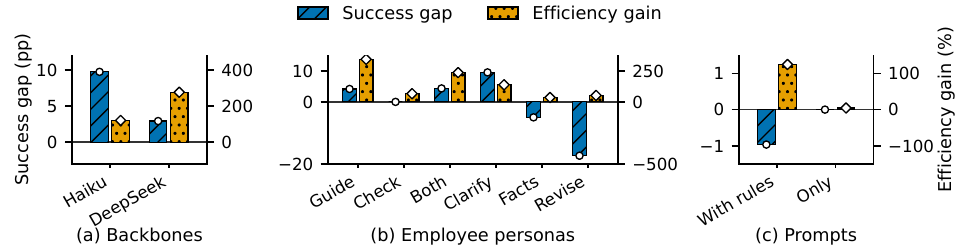}
\caption{\textbf{Token-max versus outcome-max across policy implementations.}
Blue: success gap in percentage points (left axis); gold: efficiency gain in \% (right axis).
``With rules'' states both the objective and continuation instructions;
``Only'' states the objective alone. Positive gaps favor outcome-max.}
\label{fig:robustness}
\end{figure}

Figure~\ref{fig:robustness}(a,b) shows higher outcome-max efficiency across
both backbones and all six personas, but success varies. Revise favors
iterative correction, which may benefit from token-max's continuation rules;
the uncertain success gap prevents a causal interpretation. Panel (c) shows
smaller gaps when prompts state only the objective.

Objective-only prompts show a small directional efficiency gain and unchanged
mean success, but the task-cluster interval spans both gains and losses.
This condition therefore suggests a possible role for objective framing
without establishing a reliable improvement. Appendix~\ref{app:paired-uncertainty}
reports the intervals; model and persona differences also prevent attributing
the cross-setting gap solely to wording. The main comparison supports a
clearer conclusion about explicit continuation and stopping policies than
about the behavioral effect of changing an objective alone.

\subsubsection{RQ3: Why is outcome-max more efficient?}
\label{sec:rq3}
\label{sec:overeffort-mech}

In this section, we aim to locate token costs in repeated actions, learned
rules, and selection (Figure~\ref{fig:rq2-three}). Difficult tasks may require
longer interactions, so duration alone does not identify the benefit of extra effort.

\begin{figure}[!htb]
\centering
\begin{subfigure}[b]{0.25\linewidth}
\centering
\includegraphics[width=\linewidth]{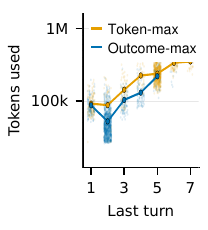}
\caption{Conversation cost}
\end{subfigure}\hfill
\begin{subfigure}[b]{0.25\linewidth}
\centering
\includegraphics[width=\linewidth]{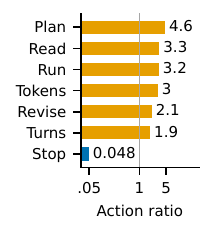}
\caption{Repeated actions}
\end{subfigure}\hfill
\begin{subfigure}[b]{0.48\linewidth}
\centering
\includegraphics[width=\linewidth]{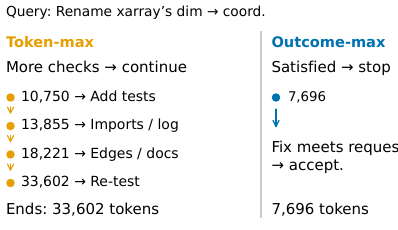}
\caption{Stopping case}
\end{subfigure}
\caption{\textbf{Where the efficiency gap arises.}
(a) Mean agent tokens by the turn at which conversations end; (b) token-max/outcome-max
ratio of action counts (above~1: token-max repeats more); (c) a paired xarray
attempt: token-max keeps requesting checks after a passing patch; outcome-max accepts.}
\label{fig:rq2-three}
\label{fig:stopping-diagnostics}
\label{fig:rq2-tactics}
\end{figure}

\begin{samepage}
Token-max conversations extend to later turns and involve more planning,
inspection, command execution, and revision. Outcome-max ends sooner when
the employee judges the delivered work sufficient. Matched attempts also
show more token-max continuation after the final correct patch is available
(Appendix~\ref{app:checkpoint-continuation}). This audit locates the earliest
saved checkpoint identical to the final correct patch and measures subsequent
interaction. It is conservative: an earlier, different patch might already
have been correct. The pattern nevertheless shows that some of the extra
review occurs after a successful artifact is in place, rather than solely
because token-max encounters more difficult work. Some later verification
may still be useful, so post-correctness tokens are not all removable waste.
\end{samepage}

The complementary failure pattern is premature acceptance. On hard tasks,
outcome-max sometimes stops with an incorrect patch, while further token-max
revision can still produce a correct result. Resource limits also constrain
some attempts, but the small budget probe does not show that raising the
limit alone repairs this problem (Table~\ref{tab:hard-stopping}). Taken
together, the audits favor a correctness-sensitive stopping policy: remove
unproductive continuation while preserving verification that can still
change the result.

\input{tables/tab_stopping_lessons.tex}

\paragraph{What experience does self-evolution preserve?}
We examine whether adaptation preserves different supervisory priorities
and relate them to human interactions (Figure~\ref{fig:learned-topics}).
Panel (a) shows persistent emphasis on context, targeted fixes, and stopping
under outcome-max, versus explanations, alternatives, and repeated review
under token-max. Objective-conditioned reflection preserves these priorities;
the libraries do not demonstrate spontaneous adoption of the objectives.

Panel (b) relates these priorities to SWE-chat user personas. Outcome-max
and vague requesters tend toward shorter interactions with higher success,
whereas token-max and mind changers tend toward longer interactions with
lower success. The measures differ across populations, so this alignment is
descriptive. Together, the panels suggest a broader implication for AI
literacy: evaluate targeted problem solving and necessary verification
instead of treating interaction volume as evidence of productive effort.
This does not identify workplace incentive effects.

\begin{figure}[!htb]
\centering
\begin{subfigure}[b]{0.62\linewidth}
\centering
\includegraphics[width=\linewidth]{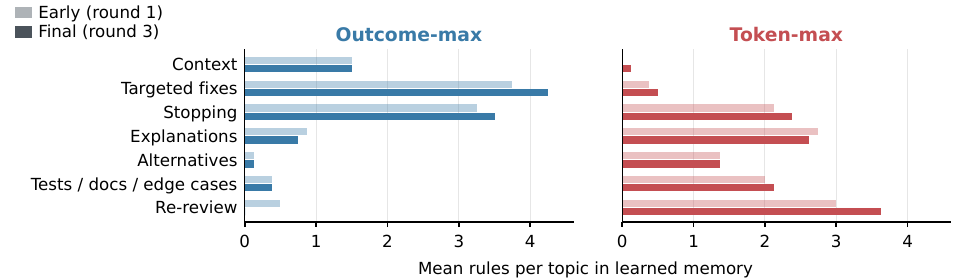}
\caption{Learned rules by topic}
\end{subfigure}\hfill
\begin{subfigure}[b]{0.34\linewidth}
\centering
\includegraphics[width=\linewidth]{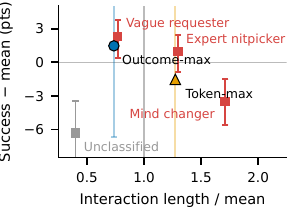}
\caption{Human personas vs.\ objectives}
\end{subfigure}
\caption{\textbf{Learned experience and interaction depth.}
(a) Employee playbook topics after early (light) and final (dark) learning;
outcome-max emphasizes targeted work and stopping. (b) SWE-chat personas
\citep{baumann2026swechat} and simulated objectives compared by normalized
interaction length and mean-centered success. Human ratings and simulated
pass rates are distinct measures; bars show 95\% cluster-bootstrap intervals.}
\label{fig:learned-topics}
\end{figure}

\paragraph{Which employees survive, and how does profit change?}
We examine how selection connects supervisory styles to firm performance
(Figure~\ref{fig:l3-dynamics}). Both firms replace the bottom three employees after
each period. Token-based selection expands mode-setting and clarification
lineages, while outcome-based selection preserves example-providing and
missing-context lineages. Revenues remain comparable, but token-max's rising
costs deepen losses while outcome-max stays profitable in the stated
accounts. These results are conditional on the imposed retention rules,
not voluntary policy adoption.

\begin{figure}[!htb]
\centering
\includegraphics[width=0.9\linewidth]{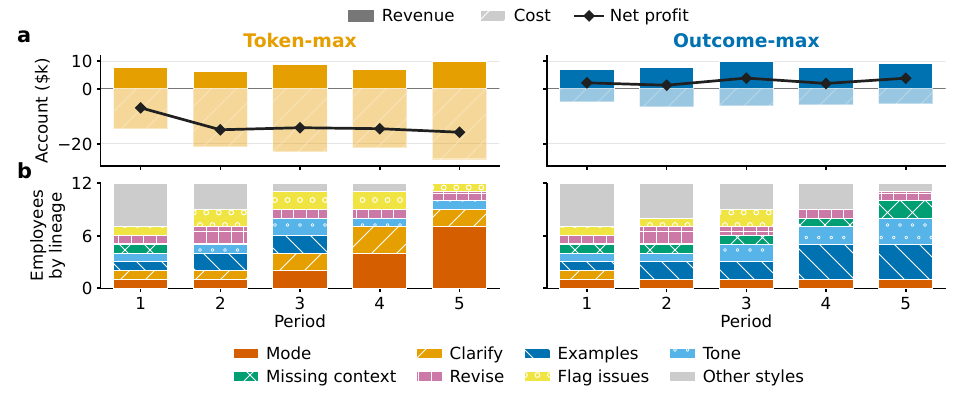}
\caption{\textbf{Workforce selection and firm performance over five periods.}
Left: token-max firm; right: outcome-max firm.
(a) Revenue (bars above zero), cost (bars below zero), and net profit (line)
under common scenario accounts.
(b) Twelve working employees per period, colored by ancestral profile.
Each period ends with three replacements from top-ranked lineages; expanding
segments include descendants.}
\label{fig:l3-dynamics}
\end{figure}

\subsection{Evaluating \MethodName}
\label{sec:mechanism-experiments}

\subsubsection{RQ4: How do heuristic contracts affect delivery and payoffs?}
\label{sec:rq4}\label{sec:experiment-winwin}
In this section, we aim to evaluate delivery and payoffs when \MethodName\
is applied to fixed production under specified inspection assumptions. Table~\ref{tab:payoffs}
compares token billing with outcome-based payments and damage-matched liability
under a common task-value and salary scenario.

\paragraph{Delivery and accounting design.}
We apply alternative submission rules to already generated artifacts, keeping
their correctness and production costs fixed. Weekly attempted work is limited
by both agent throughput and employee review capacity; rejecting an artifact
does not recover that capacity. A hypothetical inspection signal yields a
calibrated correctness probability. Under liability, the employee submits
only when the expected bonus outweighs the expected fine. Token-billing
baselines deliver all artifacts; alternative rules use confidence thresholds
or random screening. The signal is a stylized Gaussian observation model
with the observed arm-level success rate as its prior, not an empirically
trained verifier. Exact calibration and accounting assumptions remain in
Appendix~\ref{app:rq4-delivery}.

Provider fees and employee bonuses are shares of verified delivered value,
with employee salary paid separately. Defective deliveries incur losses to
the firm, while collectible fines transfer value from the employee to the
firm. Those transfers cancel in aggregate; defect losses do not. Provider
accounts measure revenue before serving costs. This separates producing
more correct work, controlling defective shipments, and distributing the
resulting value among participants.

\paragraph{Delivery tradeoffs and shared gains.}
Under token billing, outcome-max increases correct production but also ships
more defects, leaving the firm unprofitable. \MethodName\ makes delivery more
selective: it sacrifices some correct outputs while removing a much larger
share of defects. Relative to token-max, it delivers more correct work with
fewer defects and improves provider revenue, employee income, and firm profit.
The employee-heavy allocation gives the highest firm-profit estimate among
the tested \MethodName\ rules (Appendix~\ref{app:rq4-delivery}).

The employee-heavy allocation is selected from the tested contracts by its
firm-profit point estimate; it is not an independently validated optimum.
Its higher bonus changes the employee's willingness to deliver, so allocation
and delivery must be assessed jointly rather than treating sharing as a
transfer applied after a fixed submission decision.

Overall, \MethodName\ combines efficient production with defect control and
shared gains in this scenario. Its advantage depends on selective delivery
and sufficient task value, rather than maximizing shipment volume alone.
Valuation sensitivity shows that the same delivery behavior can leave the
firm profitable or loss-making under different task values. Conditional
task-cluster intervals also include firm losses at the main valuation.
The result illustrates a feasible scenario for shared gains, not robust
profitability or observed adoption of the complete contract.

\input{tables/tab_payoffs.tex}

\subsubsection{RQ5: How do acceptance, fees, and liability affect outcomes?}
\label{sec:rq5}

In this section, we aim to distinguish how acceptance, fees, and liability
affect delivery and payoffs (Figure~\ref{fig:rq4-three}). Delivery rates use
all attempted tasks as their denominator.

\begin{figure}[!htb]
\centering
\begin{subfigure}[t]{0.235\linewidth}
\centering
\includegraphics[width=\linewidth]{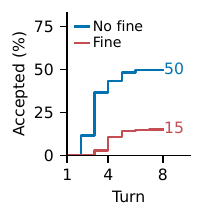}
\caption{Employee acceptance}
\end{subfigure}\hfill
\begin{subfigure}[t]{0.235\linewidth}
\centering
\includegraphics[width=\linewidth]{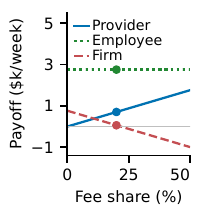}
\caption{Fee allocation}
\end{subfigure}\hfill
\begin{subfigure}[t]{0.235\linewidth}
\centering
\includegraphics[width=\linewidth]{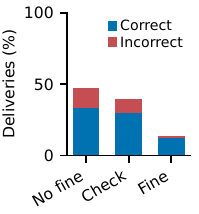}
\caption{Delivered outputs}
\end{subfigure}\hfill
\begin{subfigure}[t]{0.235\linewidth}
\centering
\includegraphics[width=\linewidth]{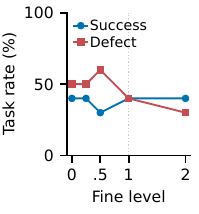}
\caption{Fine level $\lambda$}
\end{subfigure}
\caption{\textbf{Acceptance, allocation, and defect control.}
(a) Cumulative acceptance by turn with and without fine instructions;
(b) weekly payoffs at fixed production versus fee share $\phi$; (c) correct and
defective deliveries on matched policies; (d) success and defects versus fine level $\lambda$.}
\label{fig:rq4-three}
\end{figure}

\paragraph{Does liability change acceptance?}
Fine-conditioned policies show later, less frequent acceptance
(Figure~\ref{fig:rq4-three}a), consistent with more selective approval under
the combined liability and verification instructions.

\paragraph{Does the outcome fee affect shared gains?}
At fixed production, higher fees raise provider revenue and lower firm profit;
employee income is unchanged (Figure~\ref{fig:rq4-three}b). Shared gains require
satisfying all participation constraints (Eq.~\eqref{eq:feasible-fee}).

\paragraph{Does verification improve delivered quality?}
On matched attempts, checking and liability reduce both defective and correct
deliveries (Figure~\ref{fig:rq4-three}c; Appendix~\ref{app:liability-personas}).
This is a quality--quantity tradeoff: fewer defective shipments can come at
the cost of withholding correct work. The separate inspection-sensitivity
analysis makes the complementarity clear: more informative signals preserve
more correct deliveries while screening defects, improving firm returns
under the same allocation (Figure~\ref{fig:rq4-delivery-sensitivity}).
Reliable, affordable verification is therefore part of the mechanism's
practical requirement, not a consequence of token efficiency alone.

\paragraph{Does a larger fine improve outcomes?}
With learned policies fixed, fine and verification instructions yield
non-monotonic success and defect rates (Figure~\ref{fig:rq4-three}d).
The behavioral sweep changes the stated liability and verification
instructions; it does not execute monetary deductions. Its non-monotonic
responses therefore do not identify an optimal fine
(Appendix~\ref{app:fine-sweep}). This evidence is distinct from the
fixed-production contract calculation above: the former measures prompted
acceptance behavior, whereas the latter applies an explicit submission rule
and accounting model to fixed outputs.


%% file: tables/tab_rq1.tex
\begin{table}[!htb]
\centering
\caption{\textbf{Performance across three levels of human--AI simulation.}
Bold marks the better point estimate within each pair; $\dagger$ denotes reconstructed
math difficulty strata (Appendix~\ref{app:difficulty}).}
\label{tab:rq1}
\label{tab:rq1-l1}
\label{tab:rq1-l2}
\label{tab:rq1-l3}
\small
\compacttable
\setlength{\tabcolsep}{2pt}
\begin{tabular*}{\linewidth}{@{\extracolsep{\fill}}llrrrrrr@{}}
\toprule
 & & \multicolumn{4}{c}{Success (\%) $\uparrow$} & \shortstack{Tokens\\(k) $\downarrow$} & \shortstack{Efficiency\\gain (\%) $\uparrow$}\\
\cmidrule(lr){3-6}
Setting & Objective & Overall & Easy & Medium & Hard & & \\
\midrule
\multicolumn{8}{l}{\textbf{L1: single-round human–AI interaction}}\\
\multirow{2}{*}{SWE-bench} & Token-max & \textbf{48.8} & 62.2 & 43.9 & \textbf{39.8} & 21.1 & 0.0 \\
 & Outcome-max & 48.2 & \textbf{66.7} & \textbf{48.9} & 25.0 & \textbf{9.4} & \textbf{+120.7} \\
\cmidrule(lr){1-8}
\multirow{2}{*}{HumanEval+} & Token-max & 87.2 & 94.3 & 89.3 & \textbf{73.8} & 5.4 & 0.0 \\
 & Outcome-max & \textbf{87.5} & \textbf{95.9} & \textbf{91.4} & 69.6 & \textbf{4.7} & \textbf{+15.3} \\
\cmidrule(lr){1-8}
\multirow{2}{*}{GSM8K$^\dagger$} & Token-max & 89.8 & 93.9 & 86.5 & 88.5 & 2.4 & 0.0 \\
 & Outcome-max & \textbf{92.2} & \textbf{94.6} & \textbf{88.8} & \textbf{95.5} & \textbf{1.3} & \textbf{+86.5} \\
\cmidrule(lr){1-8}
\multirow{2}{*}{MATH-500$^\dagger$} & Token-max & 57.2 & \textbf{80.4} & \textbf{63.0} & 40.9 & 5.7 & 0.0 \\
 & Outcome-max & \textbf{57.4} & 80.1 & 61.5 & \textbf{42.1} & \textbf{4.6} & \textbf{+25.6} \\
\midrule
\multicolumn{8}{l}{\textbf{L2: multi-round human–AI interaction}}\\
\multirow{2}{*}{Employee adapts} & Token-max & 57.6 & 72.9 & \textbf{57.2} & 24.2 & 281.0 & 0.0 \\
 & Outcome-max & \textbf{60.9} & \textbf{80.3} & 56.7 & \textbf{28.3} & \textbf{78.9} & \textbf{+276.8} \\
\cmidrule(lr){1-8}
\multirow{2}{*}{Agent adapts} & Token-max & \textbf{62.4} & 78.3 & \textbf{59.8} & \textbf{32.8} & 206.8 & 0.0 \\
 & Outcome-max & 61.0 & \textbf{78.9} & 58.4 & 25.9 & \textbf{96.0} & \textbf{+110.8} \\
\cmidrule(lr){1-8}
\multirow{2}{*}{Both adapt} & Token-max & 55.8 & 72.4 & 54.3 & 21.1 & 262.7 & 0.0 \\
 & Outcome-max & \textbf{63.0} & \textbf{84.8} & \textbf{58.1} & \textbf{26.3} & \textbf{91.5} & \textbf{+224.2} \\
\bottomrule
\end{tabular*}
\par\smallskip
\begin{tabular*}{\linewidth}{@{\extracolsep{\fill}}lrrrr@{}}
\toprule
\multicolumn{5}{l}{\textbf{L3: organizational human--AI collaboration}}\\
Objective & \shortstack{Verified value\\(\$) $\uparrow$} & \shortstack{Total cost\\(\$) $\downarrow$} & \shortstack{Net profit\\(\$) $\uparrow$} & \shortstack{Normalized\\profit (\%) $\uparrow$}\\
\midrule
Token-max & 7,801.6 & 21,043.8 & -13,242.2 & -508.3\\
Outcome-max & \textbf{8,228.8} & \textbf{5,623.4} & \textbf{2,605.4} & \textbf{100.0}\\
\bottomrule
\end{tabular*}
\end{table}

%% file: tables/tab_stopping_lessons.tex
\begin{table}[!htb]
\centering
\caption{\textbf{Which follow-up work is wasteful, and which is useful?}
Two paired xarray attempts. The first-patch tests below are external;
employees cannot see them. Appendix~\ref{app:stopping-cases} gives the traces.}
\label{tab:stopping-lessons}
\small\compacttable
\renewcommand{\arraystretch}{1.08}
\setlength{\tabcolsep}{4pt}
\begin{tabular}{@{}>{\raggedright\arraybackslash}p{.13\linewidth}>{\raggedright\arraybackslash}p{.54\linewidth}>{\raggedright\arraybackslash}p{.29\linewidth}@{}}
\toprule
Objective & Employee behavior after the first patch & Verified result; total tokens\\
\midrule
\multicolumn{3}{@{}l}{\textbf{Case A: both agents' first patches already pass the task tests.}}\\
Token-max & Requests extra tests, import/changelog checks, then edge-case and documentation work. & Four passing patches; 33,602.\\
Outcome-max & Accepts the first patch and ends the interaction. & One passing patch; 7,696.\\
\multicolumn{3}{@{}p{\linewidth}@{}}{\textit{Why wasteful:} repeated checks and edits increase tokens without improving verified task correctness.}\\
\midrule
\multicolumn{3}{@{}l}{\textbf{Case B: both agents' first patches fail the task tests.}}\\
Token-max & Questions backward compatibility of existing \texttt{dim=} calls, prompting a repair. & The next patch passes; 30,901.\\
Outcome-max & Says the fix looks good and stops without raising compatibility. & The patch still fails; 4,456.\\
\multicolumn{3}{@{}p{\linewidth}@{}}{\textit{Why useful:} the additional compatibility check exposes a real defect; early acceptance misses it.}\\
\bottomrule
\end{tabular}
\end{table}

%% file: tables/tab_payoffs.tex
\begin{table}[!htbp]
\centering
\caption{\textbf{Delivery and weekly accounts.} \MethodName\ reduces defective shipments while retaining more correct deliveries and improving all three participant accounts relative to token-max. Selective delivery trades some correct-output volume for defect control; the design is described in Section~\ref{sec:rq4}, with calibration details in Appendix~\ref{app:rq4-delivery}.}
\label{tab:payoffs}
\small
\compacttable
\setlength{\tabcolsep}{3pt}
\begin{tabular*}{\linewidth}{@{\extracolsep{\fill}}lrrrrr@{}}
\toprule
 & \multicolumn{2}{c}{Deliveries / week} & \multicolumn{3}{c}{Weekly accounts (USD)}\\
\cmidrule(lr){2-3}\cmidrule(lr){4-6}
Production / contract & Correct $\uparrow$ & Wrong $\downarrow$ & Provider $\uparrow$ & Firm $\uparrow$ & Employee $\uparrow$\\
\midrule
Token-max / token billing & 9.8 & 15.6 & 2 & $-$5,429 & 2,400\\
Outcome-max / token billing & 38.3 & 61.3 & 3 & $-$14,746 & 2,400\\
Outcome-max / \MethodName & 19.4 & \textbf{2.9} & \textbf{535} & \textbf{275} & \textbf{3,443}\\
\bottomrule
\end{tabular*}
\end{table}

%% file: sections/07_conclusion.tex
\section{Conclusion}\label{sec:conclusion}

We simulated token-max and outcome-max behaviors across execution, memory
adaptation, and workforce selection. Outcome-max yields more verified work
per agent token under these settings, with quality costs on some
hard tasks. \MethodName\ identifies scenarios in which selective delivery and
value sharing benefit all three participants.

The main comparison specifies two behaviors rather than testing whether
objective optimization produces them. The small
directional gain under objective-only prompts suggests a possible role for
objective framing. Simulated participants and benchmark tasks
limit external validity, while financial gains depend on verification,
valuation, and collectible liability. Future experiments should separate
reward from behavior instructions and test real-user participation and
performance before drawing conclusions about workplace adoption.

%% file: sections/08_statements.tex
\section*{AI Use Statement}
AI tools were used to assist with manuscript drafting and language polishing, literature retrieval and related-work discovery, research ideation and execution, and the development and checking of mathematical arguments. These uses supported the research process, with the specific applications described in the paper. The authors retain full responsibility for the originality and accuracy of the manuscript, the validity of its mathematical claims, the reliability of its references, and the interpretation of its results. The simulated employee and production
agent roles, model configurations, and prompts are documented in Appendices~\ref{app:setup} and~\ref{app:prompts}.

\section*{Reproducibility Statement}
Appendix~\ref{app:setup} specifies tasks, models, simulation settings, and
accounting assumptions; Appendix~\ref{app:theory} gives proofs.
Appendix~\ref{app:additional} reports supporting analyses, and
Appendix~\ref{app:prompts} contains the prompts. The accompanying code and
recorded outputs support the reported tables and figures. Code is available at
\url{https://github.com/chenxu2-gif/TokenEcom}.

%% file: appendix/A_protocols.tex
\section{Experimental Setup and Evaluation Protocols}
\label{app:setup}

\begin{center}
\centering
\textbf{Appendix guide.} Each entry points to supporting evidence or implementation detail.\par\smallskip
\small\compacttable
\begin{tabular*}{\linewidth}{@{\extracolsep{\fill}}ll@{}}
\toprule
Question & Where to look \\
\midrule
Tasks, simulation levels, validity, and accounting & Appendix~\ref{app:setup} \\
Objective and verification incentives & Appendix~\ref{app:theory} \\
Robustness: cohorts, models, personas, budgets & Appendix~\ref{app:robustness} \\
Ablation: adaptation and verification components & Appendix~\ref{app:ablation} \\
Analysis: stopping, selection, accounts, sensitivity & Appendix~\ref{app:analysis} \\
Exact behavioral prompts & Appendix~\ref{app:prompts} \\
\bottomrule
\end{tabular*}
\end{center}

In this section, we aim to define comparable tasks, participant roles, and
resource accounts for the three simulation levels. A shared evaluation
protocol compares policy packages under common task and model settings.

\subsection{General setup}
We define a common interaction and evaluation protocol so that objective
comparisons use the same tasks, models, and resource limits.

\paragraph{Production interaction and evaluation.}
Let a task $x_j\in\mathcal{X}$ be assigned to employee $i$. A production agent
proposes actions or candidate outputs, and the employee provides feedback,
requests another production step, or stops; the completed trajectory $H_{ij}$
records the interaction. The trajectory is scored by a task-specific evaluator
$E:\mathcal{X}\times\mathcal{H}\rightarrow\{0,1\}$ external to both participants,
so $y_{ij}=E(x_j,H_{ij})$ is \emph{verified} rather than perceived completion.
The main comparison directly specifies token-max and outcome-max behaviors,
holding task, evaluator, models, and resource limits fixed within each pair.
It evaluates their consequences rather than whether optimizing an objective
produces the behaviors. The instruction sets are: \emph{token} (maximize interaction volume; hard to
satisfy), \emph{outcome} (a correct fix at minimum cost; stop once correct),
and, for the audit-fine mechanism
(Section~\ref{sec:pricing}), \emph{audit-fine} (accept
only after verification; shipped defects are audited and fined).

\paragraph{Tasks and evaluators.}
Repository software engineering uses SWE-bench Verified (500 human-validated
GitHub issues; correctness requires the gold tests to pass in an isolated
container~\citep{jimenez2024,openai_swebench_verified2024}). Program synthesis uses
HumanEval+ (164 problems, dataset-provided test programs~\citep{liu2023evalplus}).
Grade-school mathematics uses the first 250 tasks of the GSM8K test split~\citep{cobbe2021}.
Competition mathematics uses the first 250 tasks of MATH-500~\citep{hendrycks2021,lightman2023}.
Both mathematics benchmarks use boxed-answer extraction and normalized comparison.
Behavioral coverage is checked against WildChat-1M~\citep{zhao2024wildchat} coding
conversations (Appendix~\ref{app:validity}).

\paragraph{Difficulty groups.}
\label{app:difficulty}
Table~\ref{tab:rq1} fixes one grouping per benchmark before comparing objectives.
SWE-bench uses estimated human repair time: below 15 minutes (Easy), 15 minutes
to one hour (Medium), and at least one hour (Hard). The scheduled L1 grid contains
five tasks per group. L2 uses its separate
held-out task set. MATH uses official
levels 1--2, 3, and 4--5. HumanEval+ and GSM8K have no corresponding labels in
these records, so we use reference-solution complexity proxies. HumanEval+
counts Python AST statements in the canonical solution, excluding docstrings;
the groups are at most 4, 5--8, and at least 9 statements. GSM8K counts annotated
calculations in the reference answer; the groups are at most 2, 3--4, and at
least 5. These thresholds are the one-third and two-thirds empirical quantiles
of the full reference splits, with ties kept together. They are structural
proxies, not validated difficulty scales, and need not order model success
monotonically. All groups pool valid-run successes and counts directly.

Mathematics task identities are inferred from the task-major ordering of
the first 250 benchmark items and the four-persona cycle because per-attempt
IDs are unavailable. Difficulty results are conditional on this ordering;
HumanEval+ and SWE-bench instead use explicit task IDs.

\paragraph{Models and roles.}
The fixed-policy SWE-bench grid uses Claude Haiku 4.5 production agents and
Claude Sonnet 4.6 employees across 14 profiles, 15 tasks, and two repetitions.
The mathematical and program-synthesis runners specify DeepSeek v4 Pro agents
and v4 Flash employees. L2 separately uses DeepSeek v4 Pro agents, Haiku 4.5
employees, and Sonnet 4.6 reflection. These identities come from experiment
configurations; per-attempt model telemetry is unavailable.
The static L1 grid and L2 adaptation runs use different task cohorts and model
pairings, so comparing their absolute success rates does not isolate learning.
L3 is evaluated on its population-selection task pool. Its trajectories are
not included in the fixed-policy L1 row.

\paragraph{Employee personas.}
Table~\ref{tab:personas} defines the employee behaviors adapted from AI Fluency
\citep{dakan2025aifluency}; Appendix~\ref{app:prompts} gives their prompt encoding.

\begin{table}[!htb]
\centering
\caption{\textbf{Employee behavior dimensions.} The first six dimensions
control description; the remaining five control discernment. Low and high
archetypes use propensities of 0.1 and 0.9; single-behavior profiles elevate
one dimension.}
\label{tab:personas}
\small
\compacttable
\begin{tabular*}{\linewidth}{@{\extracolsep{\fill}}lll@{}}
\toprule
\textbf{Dimension} & Low propensity & High propensity\\
\midrule
\textbf{Iterates and refines} & Accepts initial output & Requests revisions\\
\textbf{Clarifies goals} & Leaves goals implicit & States goals and constraints\\
\textbf{Provides examples} & Gives no examples & Supplies examples\\
\textbf{Specifies format} & Leaves format open & Specifies output structure\\
\textbf{Sets interaction mode} & Uses no explicit role & Assigns a collaboration role\\
\textbf{Communicates tone} & Leaves style open & States style preferences\\
\midrule
\textbf{Identifies missing context} & Omits missing information & Supplies needed context\\
\textbf{Questions reasoning} & Accepts explanations & Challenges reasoning\\
\textbf{Fact-checks outputs} & Trusts claims & Checks correctness\\
\textbf{Flags issues} & Leaves errors unchallenged & Points out errors\\
\textbf{Requests alternatives} & Accepts one approach & Requests other approaches\\
\bottomrule
\end{tabular*}
\end{table}

\paragraph{Measurement.}
With verified outcome $y_{ij}\in\{0,1\}$, we price recorded model usage
according to Eq.~\eqref{eq:cost}. Main-table efficiency consistently counts
production-agent input and output tokens. Diagnostics that include employee tokens are labeled separately. L3 usage
contains input and output counts without separate cache measurements; its
accounts therefore use two-component input/output pricing. Human supervision cost is
$C^{\mathrm{human}}_{ij}=w\,h_{ij}$, with loaded wage $w$ and a scenario-based
supervision-time proxy $h_{ij}$ (anchored at 15~min per attempt, scaled by
persona-token burden); dollar levels are therefore scenario results, whereas
token use and verified correctness are directly measured. Over the valid runs
$\mathcal V$ we report the resolution rate $\bar y$, mean tokens $\bar T$, and
the outcome efficiency of \eqref{eq:eta} (resolved tasks per million
tokens). The two organizational key-performance indicators used for evaluation
and selection, $s^{\mathrm{tok}}_i$ and $s^{\mathrm{out}}_i$, are defined in
\eqref{eq:kpi}, where $\rho_i\in[0,1]$ is the mean parsimony of resolved patches (a
$0.9$--$1.1$ tie-breaker toward smaller patches). Cost-reduction ratios are
computed within each domain and never averaged across domains.

We study three levels of increasing organizational complexity, sharing the setup
above.

\subsection{Level-1 simulation: supervised production}
In this section, we aim to measure the consequences of specified supervision
policies within a task.
A single supervised interaction per task supports the token-versus-outcome
comparison and the over-efforting analysis (Section~\ref{sec:overeffort}). The
cross-task comparison spans all four task families; the balanced \emph{static
persona grid} crosses 14 profiles $\times$ 15 tasks $\times$ 2 repetitions per
objective (420 scheduled runs per arm before validity filtering), so every persona faces the same tasks in
both arms. A separate Haiku validation set evaluates frozen generation-5
learned policies on 100 difficulty-stratified SWE-bench tasks with two
repetitions; it is not an extension of the fixed-policy L1 grid
(Table~\ref{tab:backbone-objectives}). A persona-only model grid crosses three
agent models with two employee models over four profiles $\times$ 15 tasks. The two mathematics datasets each use 250
problems $\times$ 4 personas (2{,}000 runs across both datasets) and HumanEval+ 164 $\times$ 4 (656
runs), with at most 10 employee exchanges; the software runs allow at most 8
exchanges and 50 internal agent turns per exchange.

\subsection{Level-2 simulation: human--agent co-evolution}
In this section, we aim to test whether the policy gap persists as
participants learn reusable policies. Both sides adapt across generations for the skill-learning results
(Section~\ref{sec:entrench}). Curation uses 20 training and 15 disjoint held-out
SWE-bench tasks; the protocol runs for three generations under each objective
and in three regimes that vary who adapts (employee-only, agent-only, or
both), with per-profile runs over 11 starting profiles. The learning
\emph{protocol} itself is defined in Appendix~\ref{app:methods}
(Algorithm~\ref{alg:bilateral}); frozen generation-3 skills are evaluated on the
held-out set and, for the efficiency panel, on the 100-task validation set.

\subsection{Level-3 simulation: population selection (the firm)}
In this section, we aim to evaluate workforce changes under specified
retention scores and repeated selection. A firm selects its workforce on a leaderboard for the firm results
(Section~\ref{sec:entrench}). Each firm begins with $n=12$ employees (one
all-low baseline plus one per elevated behavior dimension), assigns 20 backlog
tasks per employee per period, and runs for 5 organizational periods (1{,}200
launched runs per arm). Each firm uses its assigned supervisory policy and retention score,
$s^{\mathrm{tok}}$ versus $s^{\mathrm{out}}$ of \eqref{eq:kpi}: at the end
of period $t$ the bottom $K=3$ employees are removed and replaced by mutated
descendants of the top $K$ (one randomly selected profile dimension flipped with probability $0.2$), so the workforce updates as in \eqref{eq:selection}. Period
operating profit uses difficulty-weighted ground-truth solves (so a firm cannot
profit from false acceptance) with the stated scenario of Table~\ref{tab:pnl}.

The two selection scores are
\begin{equation}
s_i^{\mathrm{tok}}=\frac{1}{n_i}\sum_j T_{ij},\qquad
s_i^{\mathrm{out}}=
\frac{\sum_j y_{ij}}{\sum_j C^{\mathrm{AI}}_{ij}}(\alpha+\beta \rho_i),
\label{eq:kpi}
\end{equation}
where $n_i$ counts evaluated tasks and $\rho_i$ is resolved-patch parsimony.
The fixed simulation uses $\alpha=0.9$, $\beta=0.2$, $n=12$, $K=3$,
and mutation probability $\mu=0.2$; these are simulation settings, not
parameters fitted to the observed objective gap.
Selection applies
\begin{equation}
\mathcal W_{t+1}=\mathrm{survivors}_{n-K}(\mathcal W_t;s)
\cup\mathrm{mutate}(\mathrm{top}_K(\mathcal W_t;s)).
\label{eq:selection}
\end{equation}
with $K=3$; each recruited profile has probability $0.2$ of one dimension flip.

\begin{table}[!htb]
\centering
\caption{\textbf{Population-accounting assumptions.} Monetary values define
the stated scenario, not measured workplace prices or returns.}
\label{tab:pnl}
\small
\compacttable
\begin{tabular*}{\linewidth}{@{\extracolsep{\fill}}ll@{}}
\toprule
\textbf{Parameter} & Scenario value\\
\midrule
\textbf{Annual backlog} & 2,400 tasks\\
\textbf{Task-value tiers} & \$16, \$40, \$200, \$640\\
\textbf{Loaded wage} & \$80 per hour\\
\textbf{Employee-time anchor} & 15 minutes per attempt\\
\textbf{Input, output, cache prices} & \$0.27, \$1.10, \$0.07 per million tokens\\
\textbf{Employees and periods} & 12 employees; five selection periods\\
\bottomrule
\end{tabular*}
\end{table}

\subsection{Learning and Economic Accounting}
\label{app:methods}

This subsection specifies the learning protocol used in
Section~\ref{sec:entrench} and the scenario accounting used to evaluate the
mechanism in Section~\ref{sec:pricing}.

\subsubsection{Metric-conditioned bilateral verbal reinforcement}

In this section, we aim to separate learning by the employee from learning
by the production agent. Each adapts a textual policy from shared outcome
feedback, allowing comparisons across adaptation roles.

\paragraph{Learning architecture.}
We extend verbal reinforcement---in which a language model improves a
persistent text policy from experience without weight
updates~\citep{shinn2023}---from a single learner to two functionally distinct
ones: an employee library $L_g$ of interaction tactics (how to instruct,
verify, and stop) and an agent library $K_g$ of production tactics (how to
interpret, construct, and stop). All model weights stay fixed across
generations.

\paragraph{Joint policy and outcome-only feedback.}
With objective $o\in\{\mathrm{tok},\mathrm{out}\}$, persona $b$, and textual
state $Z_g=(L_g,K_g)$, an interaction yields a trajectory
$H_{gj}\sim P(\cdot\mid x_j,b,o,L_g,K_g)$. After the external evaluator runs, a
deterministic compression $A$ keeps only the outcome vector
\begin{equation}
f_{gj}=A(H_{gj})=(y_{gj},a_{gj},d_{gj},\tau_{gj},T_{gj}),\qquad F_g=\{f_{gj}:j\in\mathcal{D}^{\mathrm{train}}_g\},
\label{eq:reflectionfeedback}
\end{equation}
where $a$ is the acceptance decision, $d=a(1-y)$ a shipped-defect indicator, and
$\tau$ the exchange count; the reflection model sees $F_g$ and the current
library, but no task text, output, or transcript.

\paragraph{Role-specific, metric-conditioned updates.}
Two operators transform the common feedback,
\begin{equation}
L_{g+1}=R_H(L_g,F_g;u_H^o),\qquad K_{g+1}=R_A(K_g,F_g;u_A),
\label{eq:bilateralupdate}
\end{equation}
each returning at most ten imperative tactics. The employee utility $u_H^o$ is
metric-conditioned (token reflection is asked to increase interaction volume,
outcome reflection to increase correct resolutions per token); $u_A$ asks the
agent to raise correct resolutions at fewer tokens. Both updates depend on the
same $F_g$, which was generated jointly by $L_g$ and $K_g$, so the two libraries
co-evolve through shared, externally verified feedback across a generational
barrier---distinct from multi-agent weight training. Three ablations (who
adapts: employee-only, agent-only, both) use the same rule. Frozen skills are
scored on the disjoint held-out set by
\begin{equation}
\eta^{\mathrm{test}}_{g,o}=\frac{\sum_{j\in\mathcal{D}^{\mathrm{test}}}y_{gj}}{\sum_{j\in\mathcal{D}^{\mathrm{test}}}T_{gj}/10^6};
\label{eq:heldoutlearning}
\end{equation}
because the reflection channel carries no task text, changes in
\eqref{eq:heldoutlearning} measure transfer of reusable routines, not replay.
Algorithm~\ref{alg:bilateral} summarizes the protocol.

\begin{algorithm}[t]
\caption{Metric-conditioned bilateral verbal reinforcement}
\label{alg:bilateral}
\small
\textbf{Input:} objective $o$; initial libraries $(L_0,K_0)$; fixed models;
training set $\mathcal D^{\mathrm{train}}$; held-out set
$\mathcal D^{\mathrm{test}}$; evaluator $E$; generations $G$.\\
\textbf{Output:} $\{(L_g,K_g)\}_{g=1}^{G}$ and $\{\eta^{\mathrm{test}}_{g,o}\}$.
\begin{enumerate}
\item For $g=0,\ldots,G-1$: execute $\mathcal D^{\mathrm{train}}$ under
$(L_g,K_g,o)$ to obtain $\{H_{gj}\}$.
\item Apply $E$ and $A$ to build $F_g$; drop invalid runs.
\item Update enabled libraries by \eqref{eq:bilateralupdate}; retain the
previous library if an update is empty.
\item Freeze $(L_{g+1},K_{g+1})$ and evaluate on $\mathcal D^{\mathrm{test}}$
without reflection by \eqref{eq:heldoutlearning}.
\item Pass only $(L_{g+1},K_{g+1})$ across the barrier; never pass held-out
outcomes to a reflection operator.
\end{enumerate}
\end{algorithm}

\subsubsection{Outcome-based pricing and the three-party model}

In this section, we aim to connect verified production with the payoffs
received by each participant. Throughout, the three parties are the LLM provider, firm, and employee.
The employee supervises production; the production agent is the LLM system.

\paragraph{Production and review capacity.}
Weekly production is limited by whichever resource runs out first: agent
computation or employee review time. If their capacities are $\mu_A$ and
$\mu_H$ tasks per week, attempted output is $M=\min(\mu_A,\mu_H)$ and verified
output is $Q=M\bar y$. Agent capacity divides its weekly output-token budget
by output tokens per task; review capacity divides available minutes by review
minutes per task. With different task values, weekly verified value is
$Z=M\,\overline{Vy}$: we value each successful task before averaging.

\paragraph{Allocating verified value.}
The outcome fee $\phi$ transfers a share of $Z$ to the provider, and the bonus
rate $B_E$ transfers another share to the employee. The firm retains what
remains after salary. Thus provider revenue is $\phi Z$, employee income is
salary plus $B_E Z$, and firm profit is $(1-\phi-B_E)Z$ minus salary.
These accounts hold observed production fixed as the fee changes. They do not
assume that a higher fee improves behavior. Provider revenue is reported before
its compute costs, so it does not measure provider profit.

\paragraph{Pigouvian audit-fine.}
A shipped defect is detected with effective probability $p$ (audit sampling
probability times audit sensitivity), and a fine $F$ is levied on its approver, with $p\cdot F=\kappa V$ so the expected fine
equals the social damage $\kappa V$ (rework cost per unit value). The equality internalizes expected defect damage in the reduced effort model
below; it does not by itself establish a particular fee optimum or a joint
improvement in all three parties' payoffs. The observed outcome and audit-fine
comparison is reported in Table~\ref{tab:ablation}b, separately from the
fixed-skill fine sweep in Table~\ref{tab:auditfine}.

In the behavioral experiments, liability is communicated through the simulated
employee's objective and, for learned policies, its reflection instructions.
The audit-fine prompt describes a large penalty for approving an incorrect
fix. The fixed-policy sweep instead states expected penalties of
$0,0.25,0.5,1,$ or $2$ times the value of a correct solve. These runs evaluate
acceptance and correctness; they do not separately sample audit events or
execute monetary deductions. Thus they test responses to stated liability,
while the contract equation specifies how an implemented audit and fine would
produce that expected liability. Identifying the sweep's multiplier with
$\lambda=pF/(\kappa V)$ uses the explicit normalization $\kappa=1$: defect damage
is valued at one correct solve. This calibration is an accounting assumption,
not an estimated damage level.

\paragraph{Three-party payoff gains.}
Weekly payoff gains apply the accounting of Table~\ref{tab:pnl} per
employee seat to the valid labor-market runs of Section~\ref{sec:entrench}
(token arm versus outcome arm, all generations). Weekly throughput is the tandem
queue above with an agent seat of 40~h/week at 50 output tokens/s and 30~h/week
of supervision, where a task's supervision minutes scale the 15-minute anchor by
its persona-token burden; the employee salary is
$30\,\mathrm{h}\times\$80=\$2{,}400$ per week. Under token pricing the LLM provider
bills tokens at the agent prices of Table~\ref{tab:pnl}, the employee earns
the salary, and the firm books the value of resolved fixes minus the bill and
the salary; under outcome pricing (fee $\phi=0.20$, employee bonus $10\%$ of
each verified fix) the LLM provider earns $\phi M\overline{Vy}$, the employee salary plus
bonus, and the firm the remainder. Confidence intervals jointly resample
the 50 shared tasks across arms (20{,}000 resamples, 95\% percentile intervals),
preserving within-task dependence across employees and periods. Calibration
and capacity parameters stay fixed; these intervals do not measure variation
across independently simulated firms.

The verification guarantee and its assumptions are proved in
Appendix~\ref{app:theory}; the accounting above does not assume that the
complete contract has been jointly deployed.

\subsection{Behavioral Coverage}
\label{app:validity}
In this section, we aim to assess measurement validity and behavioral
coverage through common exclusion rules, task-cluster uncertainty, and
comparisons with human conversations.

\paragraph{Behavioral diagnostics.}
Termination records distinguish employee acceptance from resource-limit or
agent termination. Stored follow-ups are classified by keyword rules into
requests for tests, alternatives, explanation, scope expansion, and questions,
plus an approve-and-stop category. Persona comparisons use the balanced grid;
tool and turn-cost summaries use their stated run cohorts. Supervision minutes
are a scenario proxy rather than measured human labor. These diagnostics are
descriptive mediators, not independently randomized interventions.

\paragraph{Coverage of human behavior.}
Prior studies assess how closely simulated populations reproduce human behavior
\citep{argyle2023out,aher2023turing}. We assess behavioral coverage by comparing
simulated interactions with WildChat-1M~\citep{zhao2024wildchat} coding conversations.
A high-recall lexical filter selects candidate
multi-turn conversations, followed by LLM coding validation. Follow-up style,
supervision mode, and perceived completion are annotated without the treatment
label using a fixed annotation seed. Figure~\ref{fig:alignment-depth} shows overlap in
interaction depth. Simulated personas are more decisive and less likely to
begin with underspecified requests or expand scope than observational users.
WildChat lacks executable correctness labels and randomized objectives; this
check assesses behavioral coverage, not workplace productivity effects.

\begin{figure}[!htb]
\centering
\begin{subfigure}[t]{0.32\linewidth}
\centering
\includegraphics[width=\linewidth]{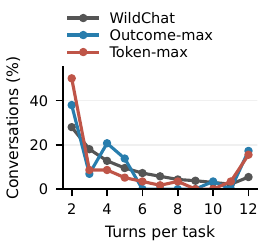}
\caption{Conditional user-turn distribution.}
\end{subfigure}\hspace{0.04\linewidth}
\begin{subfigure}[t]{0.32\linewidth}
\centering
\includegraphics[width=\linewidth]{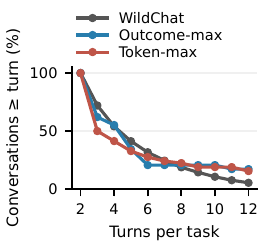}
\caption{Depth in the multi-turn tail.}
\end{subfigure}
\caption{\textbf{Interaction-depth coverage.} Both panels compare WildChat coding conversations with simulated interactions,
conditioned on multiple user turns. The tail view groups turns 2--12; it is not
an additional sample. Overlap supports behavioral coverage, not causal
productivity validity.}
\label{fig:alignment-depth}
\end{figure}

%% file: appendix/B_theory.tex
\section{Objective Properties and Verification Incentives}
\label{app:theory}

In this section, we aim to establish how objectives govern continuation
and how liability aligns verification (Theorems~\ref{prop:continuation}
and~\ref{prop:pigou}). The analysis concerns the objectives and mechanism in
Sections~\ref{sec:setup} and~\ref{sec:pricing}. Neither guarantee assumes that
prompted LLMs reach an optimal policy, and neither proves that outcome-max
improves resolution on every task distribution.

\subsection{Unproductive continuations}
In this section, we aim to characterize the response to additional computation,
first when verified output is unchanged and then when it improves. The ratio
of expected output to expected cost measures efficiency over the workload;
averaging per-task output-to-cost ratios would instead give more weight to
cheap tasks.
\begin{proof}[Proof of Theorem~\ref{prop:continuation}]
Consider two feasible policies $\pi,\pi'$ on the same task distribution.
Suppose $\pi'$ adds a continuation with no change in expected verified
resolution: $\mathbb E_{\pi'}[y]=\mathbb E_{\pi}[y]=Y>0$. If expected tokens
increase by $\Delta T>0$ and expected monetary compute cost increases by
$\Delta C>0$, with $C=\mathbb E_{\pi}[C^{\mathrm{AI}}]>0$, then
\begin{equation}
J_{\mathrm{tok}}(\pi')-J_{\mathrm{tok}}(\pi)=\Delta T>0,
\qquad
J_{\mathrm{out}}(\pi')-J_{\mathrm{out}}(\pi)
=-\frac{Y\Delta C}{C(C+\Delta C)}<0.
\label{eq:continuation-proof}
\end{equation}
Substitute the two policies' expected token counts in
Eq.~\eqref{eq:token-objective}. For Eq.~\eqref{eq:outcome-objective}, subtract
$Y/C$ from $Y/(C+\Delta C)$ and combine denominators. All factors in the
resulting denominator and numerator are positive.

If $Y=0$, both outcome-efficiency scores are zero, so the strict conclusion does
not apply. More generally, if the continuation changes expected resolution by
$\Delta Y$, subtracting $Y/C$ from $(Y+\Delta Y)/(C+\Delta C)$ gives
$(C\Delta Y-Y\Delta C)/(C(C+\Delta C))$. Since the denominator is positive,
\begin{equation}
J_{\mathrm{out}}(\pi')>J_{\mathrm{out}}(\pi)
\quad\Longleftrightarrow\quad
\Delta Y\, C>Y\,\Delta C.
\label{eq:marginal-continuation}
\end{equation}
\end{proof}

For positive $\Delta C$, this compares the marginal resolution gain per dollar
with the current average. It explains why outcome-max may favor additional
reasoning when useful; it does not prescribe universally shorter trajectories.
The theorem is a comparison of policies with stipulated expectations. The
experiments establish whether observed continuations have the relevant
properties, rather than assuming that all extra tokens are wasteful.

\subsection{Pigouvian verification alignment}
In this section, we aim to identify when an audit-fine internalizes expected
damage from shipped defects. The result concerns the verification decision
under the assumptions stated below.
\begin{proof}[Proof of Theorem~\ref{prop:pigou}]
Let $\mathcal E$ be a common nonempty feasible set of verification efforts and
suppose both minimization problems below attain their minima. Let $c(e)$ be the
effort cost and $\delta(e)$ the probability of a shipped defect. Assume:
(i) the employee is risk neutral; (ii) audits detect a shipped defect with
fixed effective probability $p>0$ and the fine $F$ is collectible;
(iii) each shipped defect causes external damage $\kappa V$; and
(iv) all remaining private and social payoff terms are constant in $e$.
Define
\begin{equation}
L_{\mathrm{private}}(e)=c(e)+pF\delta(e),\qquad
L_{\mathrm{social}}(e)=c(e)+\kappa V\delta(e).
\end{equation}
Then $pF=\kappa V$ implies
\begin{equation}
\arg\min_{e\in\mathcal E}L_{\mathrm{private}}(e)
=\arg\min_{e\in\mathcal E}L_{\mathrm{social}}(e).
\label{eq:pigou-alignment}
\end{equation}
Substitution makes the two loss functions equal at every feasible
$e$. Their ordering of all feasible efforts and hence their sets of minimizers
are identical. Fines are transfers, so they are not counted as additional
social damage.
\end{proof}

This calibration is sufficient, but need not be necessary: distinct coefficients
can share a boundary minimizer or give identical incentives when $\delta$ is
constant. With differentiable costs and an interior common optimum $e^*$ at
which $\delta'(e^*)\ne0$, subtracting the two first-order conditions does imply
$pF=\kappa V$. Existence, uniqueness, and interiority require additional assumptions.

\paragraph{Scope and implementation limitations.}
If effort changes a private bonus $b(e)$, the private objective becomes
$c(e)+pF\delta(e)-b(e)$. If effort also changes throughput or social production
value, the social objective changes as well. Eq.~\eqref{eq:contract} alone then
need not align the full objectives. Likewise, audit expense, risk aversion,
limited liability, and imperfectly measured damages alter implementation. The
empirical audit-fine intervention is an instruction to a simulated employee;
it is evidence about behavioral responses to that instruction, not a test of
monetary incentive compatibility with human participants. Fixed-production
fee accounts allocate observed value; the theorem does not establish that
those accounts yield gains for every participant.

%% file: appendix/C_results.tex
\section{Supplementary Experiments and Analysis}
\label{app:additional}
\label{app:setting-summary}
In this section, we aim to assess the scope and consequences of the token-max/outcome-max policy
comparison. We ask whether the
policy comparison is robust across settings (Section~\ref{app:robustness}),
which design components change behavior (Section~\ref{app:ablation}), and
what trajectories and accounts explain the observed outcomes
(Section~\ref{app:analysis}). Throughout, success
means external correctness, tokens count production-agent input plus output,
and efficiency $\eta$ means verified solutions per million agent tokens.
Acceptance and bad delivery retain the separate definitions from
Section~\ref{sec:setup}. Scenario accounts follow the main tables.

\subsection{Robustness: Where Does the Objective Contrast Hold?}
\label{app:robustness}
\label{app:production-robustness}
\label{app:personas}
In this section, we aim to test robustness across personas, model pairings,
interaction budgets, and repositories. These comparisons distinguish
consistent resource savings from setting-dependent effects on correctness.

\subsubsection{Task-level uncertainty and matched success}
\label{app:paired-uncertainty}
In this section, we aim to distinguish policy gaps from task-sampling
variation through matched comparisons.

Table~\ref{tab:paired-uncertainty} reports policy gaps on shared
task--persona--repetition cells, with 95\% task-cluster bootstrap intervals.
Resampling tasks keeps repeated attempts together. The primary comparison
contains 13 tasks; the hard-task and objective-only cohorts offer less
independent evidence. These intervals describe uncertainty within the
observed tasks, rather than variation across new learning histories.
Matching also explains small differences from the full-cohort main table.

Efficiency uses total successes divided by total agent tokens. Rare bootstrap
draws with no successes in either arm are conservatively bounded when forming
relative-gap intervals, rather than discarded. Boldface in the main table
marks numerical comparisons, not statistical significance. Overall,
task matching supports the efficiency advantage, while uncertain success
gaps and hard-task losses limit claims of uniform quality improvement.
\input{tables/tab_paired_uncertainty.tex}

\subsubsection{Objective gaps across backbones, personas, and prompts}
\label{app:robustness-gaps}
In this section, we aim to determine whether models, personas, and prompt
wording change the policy contrast. Comparisons are paired within each
setting; cross-setting differences do not isolate any single factor.

Figure~\ref{fig:robustness} and Table~\ref{tab:robustness-means} compare
policy packages within each setting. Panel (a) shows higher efficiency and
success under both model pairings. A stricter common-task comparison retains
the efficiency advantage but yields opposite success gaps across pairings.
Resource savings are therefore more stable than correctness gains.

Panel (b) shows the same distinction across supervision styles: all six
profiles save resources, but fact checking and revision lose success.
Revise's emphasis on iterative refinement may fit token-max's continuation
instructions better; the uncertain success gap supports this as a possible
explanation, not a general persona ranking.

Panel (c) shows a small positive efficiency estimate and unchanged success
when prompts state only the payoff. The matched efficiency gain is 4.5\%,
with a 95\% task-cluster interval of $[-31.9,62.4]$\%. This is a directional
signal, not a precisely established effect. Differences in models and
personas also prevent isolating the effect of prompt wording. Taken together, these comparisons
support more stable efficiency gains than correctness gains for explicit
policy packages. The objective-only condition does not establish that
changing rewards alone generates the instructed supervisory styles.
\input{tables/tab_robustness_means.tex}

\subsubsection{Model and persona scope}
\label{app:backbone-scope}
In this section, we aim to assess the range of models and supervisory
styles over which resource use changes. Persona and learned-policy
comparisons use distinct cohorts.

Tables~\ref{tab:model-pairings} and~\ref{tab:model-pairings-haiku} compare
four behavior profiles across six agent--employee configurations. Stronger
description and scrutiny substantially change token use, but neither yields
uniformly better success. These profiles carry no objective directive, so
they compare the consequences of specified supervision styles without
identifying responses to rewards.
Persona efficiency gains use all-low as the reference; objective comparisons
use token-max. Model assignments follow the experimental configurations; these
supplementary comparisons have less complete evaluation-integrity evidence.
\input{tables/tab_model_pairings_opus.tex}
\input{tables/tab_model_pairings.tex}

Table~\ref{tab:backbone-objectives} extends the objective comparison to
learned policies and a separate validation cohort. It again shows that
resource savings need not imply a uniform success improvement. Different
learning histories and task sets prevent interpreting these results as an
isolated backbone effect. Overall, supervisory style changes resource use
across backbones without establishing a universal success ranking.
\input{tables/tab_backbone_objectives.tex}

\input{tables/tab_personas_grid.tex}

\subsubsection{Hard-task evaluation and budget sensitivity}
In this section, we aim to test whether evaluation faults or interaction
limits explain the hard-task success gap.

\paragraph{Internal turn limits.}
We next examine whether stopping and turn limits constrain useful work.
The matched hard-task gap is concentrated in xarray and pylint and retains
its direction when any single task is removed. Trajectories suggest two
contributors: outcome-max sometimes accepts an incorrect patch, and some
token-max attempts become correct only after further revision. Outcome-max
also reaches the shared turn cap more frequently
(Table~\ref{tab:hard-stopping}). These observations explain why shorter
interactions can sacrifice useful work, without isolating a causal mechanism.

A small budget probe does not show that a higher cap resolves the problem:
all four patches fail, and the higher-cap runs finish without using the extra
budget. Incomplete termination records and unavailable test feedback limit
this diagnostic. Overall, the quality gap survives integrity checks, and
extra permitted turns alone do not resolve premature stopping.
\input{tables/tab_hard_stopping.tex}

\subsubsection{Matched adaptation outcomes}
\label{app:matched-l2}
In this section, we aim to test whether unequal task coverage explains
the adaptation results.

Table~\ref{tab:matched-l2} compares objectives on shared task--persona
cells. Outcome-max remains more efficient, while the agent-only success
interval includes zero. Restricting all adaptation regimes to common cells
retains the small success gap, so unequal coverage does not fully explain it.
The evidence supports a resource-use difference more clearly than a success
ranking.
\input{tables/tab_matched_l2.tex}

The repository-stratified software comparison also retains the direction of
the resource-use gap across the tested codebases. This is a scope check on
the observed repositories, not an estimate for unseen codebases. Together,
matching and repository stratification support a resource-use difference
more clearly than a success ranking.
\FloatBarrier

\subsection{Ablation: Which Components Change Behavior?}
\label{app:ablation}
In this section, we aim to distinguish the effects of adaptation from
verification. The former changes learned policies; the latter changes which
outputs employees approve. Separate comparisons assess their contributions.

\subsubsection{Adaptation roles and neutral baseline}
In this section, we aim to examine which participant must adapt to change
resource use and success.

Table~\ref{tab:ablation} brings together adaptation regimes, a
neutral reference, and verification variants. Adaptation preserves the
resource advantage across roles; verification instead changes which outputs
employees accept. Different valid-run cohorts make the neutral result
a reference point rather than an isolated estimate of supervision's benefit.
Thus adaptation supports resource savings across roles, while quality
depends on how employee and agent policies develop together.
The neutral control sets all eleven behavior propensities to zero and uses
no extra objective directive. Its single exchange still begins with an
employee-generated request: all 86 valid runs record employee tokens, so it
is not a no-employee-model condition. It resolves 49 tasks on the 100-task
validation roster; unmatched valid-run coverage prevents a paired causal
comparison with the adaptation arms.

\input{tables/tab_ablation.tex}

\begin{figure}[!htb]
\centering
\begin{subfigure}[t]{0.32\linewidth}
\centering\includegraphics[width=\linewidth]{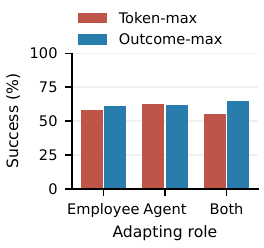}
\caption{Verified success.}
\end{subfigure}\hspace{0.04\linewidth}
\begin{subfigure}[t]{0.32\linewidth}
\centering\includegraphics[width=\linewidth]{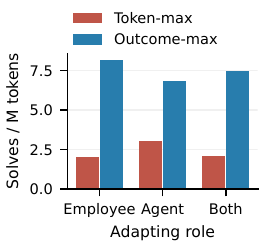}
\caption{Solves per million tokens.}
\end{subfigure}
\caption{\textbf{Adaptation-role comparison.} Each of the six
regime--objective arms is restricted to the same 327 task--persona cells,
so these matched estimates differ from the valid-run cohorts in
Table~\ref{tab:ablation}. Each panel compares objectives under a common adaptation role.}
\label{fig:supp-learning}
\end{figure}

\subsubsection{Verification liability and acceptance}
\label{app:liability-personas}
In this section, we aim to compare delivered quality under instructions
that combine liability with verification requirements.

Figure~\ref{fig:supp-verification} shows fewer defective deliveries under
fine-conditioned policies across the three supervision profiles. Interaction
length does not consistently decrease: the benefit comes from more selective
approval, rather than uniformly earlier stopping. Figure~\ref{fig:supp-verification}
shows the accompanying tradeoff---some correct outputs also remain unaccepted.
Both comparisons use shared task--persona cells within their respective
policy sets. Overall, the fine-conditioned policy packages reduce defects,
with some loss of correct deliveries rather than uniformly shorter work.
\begin{figure}[!htb]
\centering
\begin{subfigure}[t]{0.32\linewidth}
\centering\includegraphics[width=\linewidth]{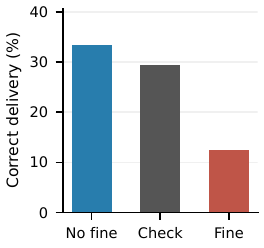}
\caption{Accepted correct outputs.}
\end{subfigure}\hspace{0.04\linewidth}
\begin{subfigure}[t]{0.32\linewidth}
\centering\includegraphics[width=\linewidth]{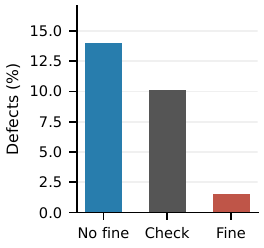}
\caption{Accepted incorrect outputs.}
\end{subfigure}
\caption{\textbf{Verification-instruction variants.} No fine, gate check,
and audit-fine use the same 129 matched task--persona cells. External
correctness is distinct from employee acceptance. These prompted components
were tested separately from the payment contract.}
\label{fig:supp-verification}
\end{figure}

\FloatBarrier

\subsection{Analysis: Why Do the Outcomes Differ?}
\label{app:analysis}
In this section, we aim to explain the objective gap through stopping,
learned behavior, selection, and economic allocation. Trajectories connect
resource use to supervision; scenario accounts connect delivered value to
participant payoffs.

\subsubsection{Stopping and learned behavior}
\label{app:diagnostics}
\label{app:learning-detail}
\label{app:agent-solo}
In this section, we aim to identify how stopping decisions and learned
rules contribute to the resource gap.

\paragraph{Trajectory diagnostics.}
We examine which activities accompany longer interactions.
Figure~\ref{fig:rq2-three} decomposes stopping, tool activity, and context
accumulation. Longer runs repeatedly process accumulated context; this is an
observed association across trajectories, not a randomized addition of turns.
Outcome-conditioned playbooks emphasize targeted fixes and stopping, while
token-conditioned playbooks emphasize explanation and continued review.
Together, the trajectory and playbook evidence links repeated review to
higher resource use, without identifying an independently causal mediator.

\paragraph{Continuing after the final correct patch.}
\label{app:checkpoint-continuation}
We test whether continued review occurs after a correct final artifact
is already available, using all matched primary SWE pairs. For each correct final patch, we identify the first saved checkpoint
with the same artifact and measure subsequent token use.
Table~\ref{tab:checkpoint-continuation} shows that token-max continues after
this point more often, with a task-cluster interval excluding zero.
This supports continued review as a source of additional cost. The measure
is conservative because earlier correct but different patches are excluded;
subsequent verification may still be useful, so later tokens are not all
removable waste.
\input{tables/tab_checkpoint_continuation.tex}

\begin{figure}[!htb]
\centering
\begin{subfigure}[t]{0.32\linewidth}
\centering\includegraphics[width=\linewidth]{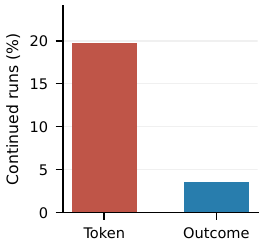}
\caption{Attempts continuing after the final correct patch.}
\end{subfigure}\hspace{0.04\linewidth}
\begin{subfigure}[t]{0.32\linewidth}
\centering\includegraphics[width=\linewidth]{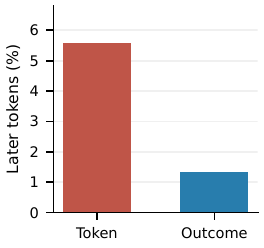}
\caption{Tokens after the checkpoint.}
\end{subfigure}
\caption{\textbf{Post-correct-patch continuation.} Both panels use the same
310 matched primary SWE pairs as Table~\ref{tab:checkpoint-continuation}.
The later-token share is descriptive, not an estimate of removable cost;
subsequent verification may still be useful.}
\label{fig:supp-stopping}
\end{figure}

\subsubsection{Paired examples of productive and premature stopping}
\label{app:stopping-cases}
In this section, we aim to distinguish productive stopping from premature
approval. The paired examples contrast extra checks after a correct patch
with a compatibility check that repairs an incorrect patch. Together, they
show that the value of continuation depends on the unresolved defect, not
on interaction length alone.
\input{tables/tab_stopping_cases.tex}
\input{tables/tab_stopping_case_notes.tex}

\paragraph{Examples of learned rules.}
We examine how learned instructions express the two objectives.
Table~\ref{tab:learned-rules} illustrates the topic contrast in
Figure~\ref{fig:learned-topics}: token-max broadens follow-up work, while
outcome-max targets fixes and stops sooner. These descriptive rules do not
establish each topic's causal effect on performance; they connect the
objective to the supervisory priorities learned from experience.
\input{tables/tab_learned_rules.tex}

\subsubsection{Population selection and three-party accounts}
\label{app:population-detail}
In this section, we aim to connect workforce selection with production
costs and the distribution of verified value.

\paragraph{Accounting scope.}
For population selection, Table~\ref{tab:firm} reports endpoint profit and
leaderboard informativeness. The accounts use the common valuation scenario in
Table~\ref{tab:pnl}; they separate verified value from production cost instead
of inferring quality from the selection score. The rank correlation relates
employee KPI rank to verified value per dollar; it is not token efficiency
$\eta$.
\input{tables/tab_firm.tex}

\paragraph{Which employees does selection retain?}
We test whether changing workforce composition accompanies the cost gap.
Figure~\ref{fig:selection-trajectories} links lower cost to a different
workforce composition. Context/example lineages expand under outcome-max
and disappear under token-max, while token-max's production costs rise.
These labels track ancestral profile families rather than measuring
post-mutation behavior. The trajectories describe one firm per objective;
Figure~\ref{fig:l3-dynamics} links these lineage changes to firm accounts.
Thus selection reinforces the styles of supervision favored by each objective.

\begin{figure}[!htb]
\centering
\begin{subfigure}[t]{0.32\linewidth}
\centering\includegraphics[width=\linewidth]{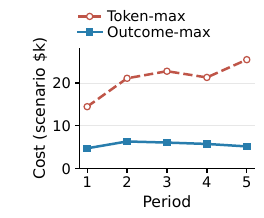}
\caption{Production and supervision cost.}
\end{subfigure}\hspace{0.04\linewidth}
\begin{subfigure}[t]{0.32\linewidth}
\centering\includegraphics[width=\linewidth]{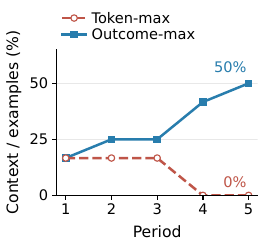}
\caption{Context/example lineages.}
\end{subfigure}
\caption{\textbf{Costs and workforce composition under selection.}
Left: API and supervision costs over the five periods.
Right: workforce share descended from example-providing or missing-context
profiles. Each firm maintains twelve employees. Main-text
Figure~\ref{fig:l3-dynamics} reports the full lineage composition and resulting accounts.}
\label{fig:selection-trajectories}
\end{figure}

The verification variants in Table~\ref{tab:ablation} complement the
accounts: audit-fines reduce defective delivery, but also reduce acceptance
and increase computation. They were evaluated separately from outcome
payments, so these results do not establish the effect of deploying both
together. Overall, lower production cost makes value sharing feasible in
the scenario, while verification and uncertain returns qualify its scope.
\paragraph{OutcomeShare case study: allocating verified value.}
\label{app:outcomeshare-case}
\begin{table}[!htb]
\centering
\caption{\textbf{How outcome payments share the value of efficient production.}
Capacity, value, and wage follow Table~\ref{tab:pnl}; these allocation-only
accounts exclude defect losses.}
\label{tab:outcomeshare-case}
\small\compacttable
\setlength{\tabcolsep}{4pt}
\begin{tabular}{@{}>{\raggedright\arraybackslash}p{.20\linewidth}>{\raggedright\arraybackslash}p{.76\linewidth}@{}}
\toprule
Component & Result and interpretation\\
\midrule
Production & Lower token use per resolution supports more attempted work under the capacity assumption.\\
Provider & Outcome fees raise revenue by sharing verified production value; revenue is before compute costs.\\
Employee & A bonus on verified output raises modeled income.\\
Firm & The residual value yields a positive profit estimate relative to the token-regime loss; its uncertainty interval still includes losses.\\
Scope & These accounts illustrate feasible value sharing, not a field deployment of the full payment-and-liability contract.\\
\bottomrule
\end{tabular}
\end{table}

\subsubsection{Fee allocation and verification sensitivity}
\label{app:fee-model}
\label{app:fee-sweep}
In this section, we aim to distinguish value allocation from changes in
delivery behavior. Fee comparisons hold production fixed, while liability
comparisons vary the employee's approval incentives.

\paragraph{How the fee divides value.}
We examine how outcome fees divide the value of efficient production.
Figure~\ref{fig:si-feesweep} shows the allocation tradeoff under fixed
production: a higher outcome fee raises provider revenue and lowers firm
profit, while employee income stays fixed. The selected fee improves the
firm's point estimate relative to token payments, but the uncertainty band
still includes losses. Thus value sharing is feasible in this scenario;
firm break-even and provider profitability are not guaranteed.
\begin{figure}[!htb]
\centering
\begin{subfigure}[t]{0.32\linewidth}
\centering\includegraphics[width=\linewidth]{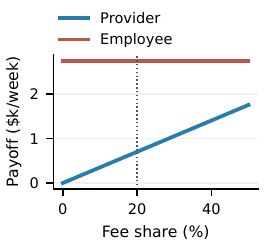}
\caption{Recipient payoffs.}
\end{subfigure}\hspace{0.04\linewidth}
\begin{subfigure}[t]{0.32\linewidth}
\centering\includegraphics[width=\linewidth]{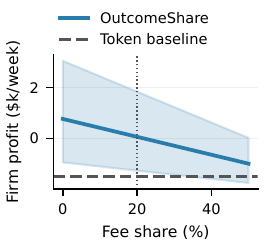}
\caption{Firm profit and uncertainty.}
\end{subfigure}
\caption{\textbf{Fee allocation under fixed observed production.}
The left panel shows recipient transfers; the right panel shows firm profit
and its pointwise 95\% task-bootstrap interval. All curves use the same
valuation and capacity assumptions. The figure is not a behavioral fee intervention
and does not identify provider net profit.}
\label{fig:si-feesweep}
\end{figure}

\paragraph{Separating efficiency from transfers.}
We separate the gains from production from the transfers between parties.
The allocation baselines in Table~\ref{tab:rq4-allocation} hold outcome-max production
fixed while changing payment rules. Token billing with fixed salary isolates
production gains; outcome billing with fixed salary isolates provider transfers;
adding the employee bonus then isolates employee sharing. These accounting counterfactuals distinguish higher production from the
distribution of its value.

\input{tables/tab_rq4_allocation.tex}
\input{tables/tab_rq4_uncertainty.tex}
Salary is reserved first, then the provider's token-baseline revenue. Equal
surplus splits the positive remainder three ways; provider--firm and
employee--firm splits divide it between the named parties. The firm bears
any shortfall. Equal gains instead divides positive incremental value over
the token baseline equally, so the firm may remain below break-even.

\begin{figure}[!htb]
\centering
\includegraphics[width=0.49\linewidth]{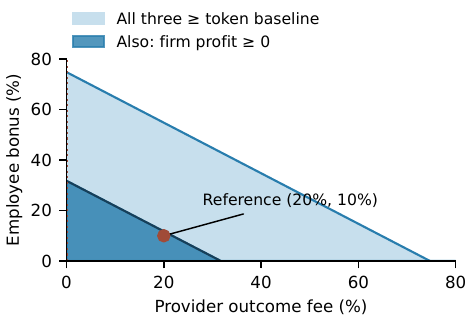}
\caption{\textbf{Participation and break-even are different requirements.}
At fixed outcome-max production, the light region preserves all three baseline
accounts; the dark region also gives the firm nonnegative profit. Provider
participation compares gross revenue, not net profit. Regions use point
estimates, excluding audit costs and defect damages. The marked reference
allocates 20\% to the provider and 10\% as employee bonus.}
\label{fig:rq4-feasibility}
\end{figure}
At the point estimates, fees plus bonuses can reach 74.7\% while preserving
the firm's loss-making baseline, but only 31.7\% while keeping it profitable.
The provider needs at least 0.062\% of verified value to match its baseline
revenue. These bounds describe allocation under the chosen valuation and
capacity assumptions. Thus preserving participation is less demanding
than achieving firm break-even, and neither follows from efficiency alone.

\paragraph{Rule-based delivery under heuristic contracts.}
\label{app:rq4-delivery}
We test whether selective delivery converts efficient production into shared gains (Table~\ref{tab:payoffs}). Synthetic inspection signals are applied to observed final artifacts, holding quality and production costs fixed; rejection does not recover capacity. Without fines, all artifacts are delivered. Under liability, risk-neutral employees compare expected bonuses with collectible fines, without an exit option.

\emph{Inspection and submission.}
Signals follow unit-variance normal distributions with means $\pm d/2$ for correct and incorrect artifacts, where $d=2\Phi^{-1}(s)$ and $s$ denotes balanced-prior signal accuracy. Bayes' rule yields correctness probability $q$ using the observed arm-level success rate as a fixed prior. This assumes ideal population calibration, not a trained verifier; gold labels generate and score signals but remain hidden from employees. Liability induces submission when $qB_EV\ge(1-q)pF$. Fixed screens require $q\ge0.5$ or $q\ge0.9$, while random screening matches the reference liability rule's delivery rate. Signal distributions are integrated analytically.

\emph{Accounting.}
All policies in Table~\ref{tab:rq4-delivery-full-high} use \$240 per reference labor-hour and a \$2,400 weekly salary. Defects incur losses equal to task value. Damage-matched fines are employee-to-firm transfers that cancel in aggregate, while defect losses remain. Firm profit deducts salary, provider fees, bonuses, and defect losses; provider revenue excludes serving costs. Shares apply to verified delivered value, with salary paid separately, unlike the preceding salary-first residual-sharing analysis.

\emph{Results and allocation.}
Outcome-max under token billing produces more correct outputs but also more defects. Confidence screening improves firm profit but lowers employee income below salary; random screening also lowers employee income and leaves the firm loss-making. Both use a 20\% provider fee and 10\% employee bonus.

We test four \MethodName\ allocations, expressed as provider/employee shares: reference (20\%/10\%), provider-heavy (40\%/10\%), equal-thirds, and employee-heavy (10\%/40\%). The employee-heavy allocation retains the most correct deliveries and achieves the highest firm-profit point estimate among these configurations, combining positive firm profit with employee income and provider revenue above token billing. It improves all five outcomes over token-max, while sacrificing some correct deliveries relative to unfiltered outcome-max to reduce defects. It is the best-performing tested allocation on firm profit, not an independently established optimum.

\input{tables/tab_rq4_delivery_full.tex}

\emph{Valuation and uncertainty.}
We test how task value affects profitability while holding the
employee-heavy allocation fixed (Table~\ref{tab:rq4-value-sensitivity}). Benefits and defect losses scale together, so
submission thresholds and delivery counts remain unchanged. Weekly firm
profit is $891.57m-2{,}400$, where $m$ scales the \$80 reference valuation;
break-even is approximately \$215 per reference hour. At \$240 the point
estimate is positive, while the lower-value \$80 scenario is loss-making.
The conditional 95\% firm-profit interval spans zero even at \$240, so the
positive point estimate does not establish robust profitability. At that
valuation, correct and wrong deliveries have respective intervals
$[12.6,27.3]$ and $[2.4,3.5]$ per week; provider revenue is
$[245,912]$ dollars and employee income is $[2,170,4,984]$.
These task-cluster intervals hold the signal model, prevalence, and prices
fixed; they exclude uncertainty about actual behavior or verifier calibration.
Thus sufficient task value supports positive expected firm returns, but
task variation prevents a claim of robust profitability.

\input{tables/tab_rq4_value_sensitivity.tex}

\emph{Inspection quality and additional checks.}
We test how inspection reliability changes the delivery tradeoff
(Figure~\ref{fig:rq4-delivery-sensitivity}). More informative signals allow
employees to retain more correct outputs while rejecting defects, improving
firm returns under the employee-heavy allocation. The sensitivity analysis
spans signal accuracy, liability, defect damage, and inspection cost.
Together, these comparisons identify reliable, affordable verification as a
condition for converting efficient production into valuable delivery.

\begin{figure}[!htb]
\centering
\includegraphics[width=\linewidth]{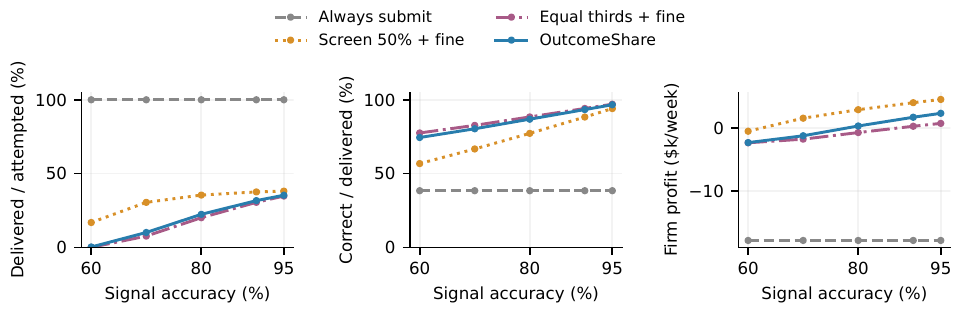}
\caption{\textbf{Inspection quality and selective delivery.}
The blue curve uses the employee-heavy \MethodName\ allocation,
with the same \$240/reference-hour valuation. Better signals improve the
quality--quantity tradeoff; alternative screens and equal-thirds sharing
illustrate dependence on the rule and allocation.}
\label{fig:rq4-delivery-sensitivity}
\end{figure}

\paragraph{How liability changes acceptance.}
\label{app:fine-sweep}
We compare acceptance under increasingly strong liability and verification
instructions.
Table~\ref{tab:auditfine} varies fine-conditioned instructions with the
persona and agent held fixed. Stronger fine-conditioned instructions generally reduce defective
delivery, but intermediate levels are not monotone and success does not
uniformly improve. The matched and full-cohort summaries show the same
basic tradeoff between scrutiny and acceptance. Because positive-fine prompts
also request verification, the sweep does not isolate the numeric penalty
or identify an optimal fine. Thus the evidence supports defect control
through scrutiny, but not the claim that larger fines consistently perform better.
\input{tables/tab_auditfine.tex}

%% file: tables/tab_paired_uncertainty.tex
\begin{table}[!htb]
\centering
\caption{\textbf{Matched success rates and task-cluster uncertainty.} Gaps are outcome-max minus token-max, with pointwise 95\% task-cluster bootstrap intervals. Bold marks the higher success rate; ties are unbolded.}
\label{tab:paired-uncertainty}
\small\compacttable
\setlength{\tabcolsep}{3pt}
\begin{tabular*}{\linewidth}{@{\extracolsep{\fill}}lrrl@{}}
\toprule
Setting & Token (\%) & Outcome (\%) & Gap [95\% CI] (pp)\\
\midrule
\multicolumn{4}{l}{\textbf{(a) Primary SWE benchmark}}\\
Overall & \textbf{51.3} & 50.3 & -1.0 [-8.2, +5.2]\\
Easy & 62.0 & \textbf{67.6} & +5.6 [+0.0, +11.4]\\
Medium & 46.8 & \textbf{50.0} & +3.2 [+0.0, +7.4]\\
Hard & \textbf{43.6} & 26.9 & -16.7 [-24.5, +0.0]\\
\midrule
\multicolumn{4}{l}{\textbf{(b) Model and persona robustness}}\\
Haiku 4.5 & 51.9 & \textbf{61.6} & +9.7 [+2.8, +16.7]\\
DeepSeek v4 Pro & 58.1 & \textbf{61.0} & +2.9 [-2.1, +7.8]\\
more guidance & 45.8 & \textbf{50.0} & +4.2 [+0.0, +13.0]\\
more checking & 57.9 & 57.9 & +0.0 [-15.0, +15.0]\\
both & 52.2 & \textbf{56.5} & +4.3 [-13.6, +22.7]\\
goal clarification & 42.9 & \textbf{52.4} & +9.5 [+0.0, +22.7]\\
fact checking & \textbf{40.0} & 35.0 & -5.0 [-22.7, +10.5]\\
iterative revision & \textbf{56.5} & 39.1 & -17.4 [-45.8, +13.6]\\
\midrule
\multicolumn{4}{l}{\textbf{(c) Prompt robustness}}\\
Objective + instructions & \textbf{51.3} & 50.3 & -1.0 [-8.2, +5.2]\\
Objective only & 37.5 & 37.5 & +0.0 [-12.0, +12.0]\\
\bottomrule
\end{tabular*}
\end{table}

%% file: tables/tab_robustness_means.tex
\begin{table}[!htb]
\centering
\caption{\textbf{Paired means underlying Figure~\ref{fig:robustness}.} Within each pair, bold marks higher success and efficiency or lower token use; ties are unbolded.}
\label{tab:robustness-means}
\small\compacttable
\setlength{\tabcolsep}{3pt}
\begin{tabular*}{\linewidth}{@{\extracolsep{\fill}}llrrr@{}}
\toprule
Setting & Objective & Success (\%) & Tokens (k) & $\eta$\\
\midrule
Haiku 4.5 & Token-max & 51.9 & 15.5 & 33.57\\
 & Outcome-max & \textbf{61.6} & \textbf{8.3} & \textbf{73.86}\\
DeepSeek v4 Pro & Token-max & 58.1 & 281.3 & 2.06\\
 & Outcome-max & \textbf{61.0} & \textbf{78.5} & \textbf{7.78}\\
\midrule
Detailed guidance & Token-max & 45.8 & 30.2 & 15.16\\
 & Outcome-max & \textbf{50.0} & \textbf{7.4} & \textbf{67.33}\\
Active checking & Token-max & 57.9 & 18.7 & 30.98\\
 & Outcome-max & 57.9 & \textbf{11.3} & \textbf{51.41}\\
Guidance + checking & Token-max & 52.2 & 35.7 & 14.61\\
 & Outcome-max & \textbf{56.5} & \textbf{11.5} & \textbf{49.14}\\
Goal clarification & Token-max & 42.9 & 20.5 & 20.86\\
 & Outcome-max & \textbf{52.4} & \textbf{10.6} & \textbf{49.61}\\
Fact checking & Token-max & \textbf{40.0} & 12.6 & 31.69\\
 & Outcome-max & 35.0 & \textbf{8.1} & \textbf{43.00}\\
Iterative revision & Token-max & \textbf{56.5} & 20.5 & 27.57\\
 & Outcome-max & 39.1 & \textbf{9.3} & \textbf{41.91}\\
\midrule
Objective + instructions & Token-max & \textbf{51.3} & 21.5 & 23.81\\
 & Outcome-max & 50.3 & \textbf{9.4} & \textbf{53.59}\\
Objective only & Token-max & 37.5 & 183.4 & 2.05\\
 & Outcome-max & 37.5 & \textbf{175.4} & \textbf{2.14}\\
\bottomrule
\end{tabular*}
\end{table}

%% file: tables/tab_model_pairings_opus.tex
\begin{table}[!htb]
\centering
\caption{\textbf{Backbone and human-persona results with an Opus employee.}}
\label{tab:model-pairings}
\small
\compacttable
\setlength{\tabcolsep}{2pt}
\begin{tabular*}{\linewidth}{@{\extracolsep{\fill}}llrrrrrr@{}}
\toprule
\midrule
\multicolumn{8}{l}{\textbf{Different backbones and human persona}}\\
\multicolumn{8}{l}{\textit{SWE-bench; human: Opus 4.6}}\\
Agent & Human behavior & Overall & Easy & Medium & Hard & Tokens (k) & Eff. gain (\%)\\
\midrule
\multirow{4}{*}{Haiku 4.5} & Minimal input & 26.7 & 20.0 & 40.0 & 20.0 & 9.8 & 0.0\\
 & Detailed guidance & 26.7 & 40.0 & 20.0 & 20.0 & 11.6 & -15.5\\
 & Active checking & 46.2 & 40.0 & 66.7 & 40.0 & 23.1 & -26.1\\
 & Guidance + checking & 28.6 & 40.0 & 25.0 & 20.0 & 25.0 & -57.8\\
\cmidrule(lr){1-8}
\multirow{4}{*}{Sonnet 4.6} & Minimal input & 42.9 & 60.0 & 40.0 & 25.0 & 12.1 & 0.0\\
 & Detailed guidance & 50.0 & 60.0 & 60.0 & 25.0 & 14.3 & -1.3\\
 & Active checking & 50.0 & 60.0 & 60.0 & 25.0 & 28.0 & -49.6\\
 & Guidance + checking & 40.0 & 60.0 & 40.0 & 20.0 & 25.5 & -55.8\\
\cmidrule(lr){1-8}
\multirow{4}{*}{Opus 4.8} & Minimal input & 33.3 & 20.0 & 40.0 & 40.0 & 10.3 & 0.0\\
 & Detailed guidance & 33.3 & 40.0 & 40.0 & 20.0 & 19.1 & -46.1\\
 & Active checking & 40.0 & 40.0 & 40.0 & 40.0 & 27.7 & -55.4\\
 & Guidance + checking & 33.3 & 40.0 & 40.0 & 20.0 & 25.2 & -59.1\\
\bottomrule
\end{tabular*}
\end{table}

%% file: tables/tab_model_pairings.tex
\begin{table}[!htb]
\centering
\caption{\textbf{Backbone and human-persona results with a Haiku employee.}}
\label{tab:model-pairings-haiku}
\small
\compacttable
\setlength{\tabcolsep}{2pt}
\begin{tabular*}{\linewidth}{@{\extracolsep{\fill}}llrrrrrr@{}}
\toprule
\midrule
\multicolumn{8}{l}{\textbf{Different backbones and human persona}}\\
\multicolumn{8}{l}{\textit{SWE-bench; human: Haiku 4.5}}\\
Agent & Human behavior & Overall & Easy & Medium & Hard & Tokens (k) & Eff. gain (\%)\\
\midrule
\multirow{4}{*}{Haiku 4.5} & Minimal input & 40.0 & 60.0 & 40.0 & 20.0 & 8.6 & 0.0\\
 & Detailed guidance & 54.5 & 60.0 & 66.7 & 33.3 & 43.2 & -72.8\\
 & Active checking & 25.0 & 40.0 & 33.3 & 0.0 & 38.4 & -86.0\\
 & Guidance + checking & 50.0 & 60.0 & 66.7 & 25.0 & 49.9 & -78.4\\
\cmidrule(lr){1-8}
\multirow{4}{*}{Sonnet 4.6} & Minimal input & 42.9 & 60.0 & 40.0 & 25.0 & 10.7 & 0.0\\
 & Detailed guidance & 42.9 & 60.0 & 40.0 & 25.0 & 25.5 & -58.3\\
 & Active checking & 40.0 & 60.0 & 40.0 & 20.0 & 41.0 & -75.8\\
 & Guidance + checking & 46.7 & 60.0 & 60.0 & 20.0 & 41.3 & -71.9\\
\cmidrule(lr){1-8}
\multirow{4}{*}{Opus 4.8} & Minimal input & 53.3 & 60.0 & 60.0 & 40.0 & 9.7 & 0.0\\
 & Detailed guidance & 46.7 & 60.0 & 40.0 & 40.0 & 33.4 & -74.5\\
 & Active checking & 46.7 & 60.0 & 40.0 & 40.0 & 64.3 & -86.7\\
 & Guidance + checking & 40.0 & 60.0 & 40.0 & 20.0 & 49.2 & -85.1\\
\bottomrule
\end{tabular*}
\end{table}

%% file: tables/tab_backbone_objectives.tex
\begin{table}[!htb]
\centering
\caption{\textbf{Objective comparisons under learned policies.}
Each model pairing uses its own training and validation protocol.}
\label{tab:backbone-objectives}
\small
\compacttable
\setlength{\tabcolsep}{4pt}
\begin{tabular*}{\linewidth}{@{\extracolsep{\fill}}llrrr@{}}
\toprule
Policy & Objective & Success (\%) $\uparrow$ & Tokens (k) $\downarrow$ & Eff. gain (\%) $\uparrow$\\
\midrule
\multicolumn{5}{l}{\textit{Haiku 4.5 agent; Sonnet 4.6 employee}}\\
Haiku 4.5 & Token-max & 51.6 & 15.5 & 0.0\\
 & Outcome-max & \textbf{60.5} & \textbf{8.4} & \textbf{+117.4}\\
\midrule
\multicolumn{5}{l}{\textit{DeepSeek v4 Pro agent; Haiku 4.5 employee}}\\
DeepSeek v4 Pro & Token-max & 55.8 & 262.7 & 0.0\\
 & Outcome-max & \textbf{63.0} & \textbf{91.5} & \textbf{+224.2}\\
\bottomrule
\end{tabular*}
\end{table}

%% file: tables/tab_personas_grid.tex
\begin{table}[!htb]
\centering
\caption{\textbf{Objective comparisons across human personas.}}
\label{tab:personas-grid}
\small
\compacttable
\setlength{\tabcolsep}{4pt}
\begin{tabular*}{\linewidth}{@{\extracolsep{\fill}}llrr@{}}
\toprule
Persona & Objective & Success (\%) & Tokens (k) \\
\midrule
\multirow{2}{*}{Guidance and checking} & Token-max & 50.0 & 34.8 \\
 & Outcome-max & \textbf{54.2} & \textbf{11.3} \\
\midrule
\multirow{2}{*}{More guidance} & Token-max & 42.3 & 29.7 \\
 & Outcome-max & \textbf{50.0} & \textbf{7.4} \\
\midrule
\multirow{2}{*}{More checking} & Token-max & \textbf{54.5} & 17.7 \\
 & Outcome-max & 52.2 & \textbf{11.4} \\
\bottomrule
\end{tabular*}
\end{table}

%% file: tables/tab_hard_stopping.tex
\begin{table}[!htb]
\centering
\caption{\textbf{Stopping reasons in the Hard cohort.}}
\label{tab:hard-stopping}
\small
\compacttable
\begin{tabular*}{\linewidth}{@{\extracolsep{\fill}}lrrrr@{}}
\toprule
 & \multicolumn{2}{c}{Token-max} & \multicolumn{2}{c}{Outcome-max}\\
\cmidrule(lr){2-3}\cmidrule(lr){4-5}
Ending & Runs & Solved & Runs & Solved\\
\midrule
Verified 50-turn cap & 18 & 5 & 46 & 6\\
Employee acceptance & 0 & 0 & 32 & 14\\
Confidence threshold & 2 & 2 & 4 & 1\\
Other agent ending & 73 & 30 & 10 & 2\\
\bottomrule
\end{tabular*}
\end{table}

%% file: tables/tab_matched_l2.tex
\begin{table}[!htb]
\centering
\caption{\textbf{Adaptation outcomes on matched tasks and personas.}}
\label{tab:matched-l2}
\small
\compacttable
\setlength{\tabcolsep}{2.5pt}
\begin{tabular*}{\linewidth}{@{\extracolsep{\fill}}llrrrrrr@{}}
\toprule
 & & \multicolumn{4}{c}{Success (\%) $\uparrow$} & \shortstack{Tokens\\(k) $\downarrow$} & \shortstack{Efficiency\\improvement\\(\%) $\uparrow$}\\
\cmidrule(lr){3-6}
Adaptation & Objective & Overall & Easy & Medium & Hard & & \\
\midrule
\multirow{2}{*}{Employee adapts} & Token-max & 58.1 & 73.4 & \textbf{57.3} & 25.0 & 281.3 & 0.0 \\
 & Outcome-max & \textbf{61.0} & \textbf{80.6} & 56.7 & \textbf{28.3} & \textbf{78.5} & \textbf{+276.5} \\
\cmidrule(lr){1-8}
\multirow{2}{*}{Agent adapts} & Token-max & \textbf{62.0} & 78.1 & \textbf{59.7} & \textbf{30.4} & 207.3 & 0.0 \\
 & Outcome-max & 61.2 & \textbf{79.6} & 58.0 & 26.8 & \textbf{94.6} & \textbf{+116.3} \\
\cmidrule(lr){1-8}
\multirow{2}{*}{Both adapt} & Token-max & 56.2 & 72.5 & 55.0 & 20.4 & 261.7 & 0.0 \\
 & Outcome-max & \textbf{64.1} & \textbf{85.5} & \textbf{59.4} & \textbf{25.9} & \textbf{88.7} & \textbf{+236.1} \\
\bottomrule
\end{tabular*}
\end{table}

%% file: tables/tab_ablation.tex
\begin{table}[!htb]
\centering
\caption{\textbf{Adaptation and verification outcomes.}}
\label{tab:ablation}
\small
\compacttable
\setlength{\tabcolsep}{3pt}
\begin{tabular*}{\linewidth}{@{\extracolsep{\fill}}llrrr@{}}
\multicolumn{5}{l}{\textbf{(a) Adapting role}}\\
\toprule
Role & Objective & Success (\%) & Tokens (k) & \shortstack{Efficiency\\(solutions per M tokens)}\\
\midrule
Neutral (zero profile) & --- & 57.0 & 54.2 & 10.51 \\
\midrule
\multirow{2}{*}{Employee} & Token-max & 57.6 & 281.0 & 2.05 \\
 & Outcome-max & \textbf{60.9} & \textbf{78.9} & \textbf{7.72} \\
\midrule
\multirow{2}{*}{Agent} & Token-max & \textbf{62.4} & 206.8 & 3.02 \\
 & Outcome-max & 61.0 & \textbf{96.0} & \textbf{6.36} \\
\midrule
\multirow{2}{*}{Both} & Token-max & 55.8 & 262.7 & 2.12 \\
 & Outcome-max & \textbf{63.0} & \textbf{91.5} & \textbf{6.89} \\
\bottomrule
\end{tabular*}
\par\medskip
\begin{tabular*}{\linewidth}{@{\extracolsep{\fill}}lrrr@{}}
\multicolumn{4}{l}{\textbf{(b) Verification objective when both roles adapt}}\\
\toprule
Objective & Success (\%) & Acceptance (\%) & Bad delivery (\%)\\
\midrule
Token-max & 55.8 & 0.3 & \textbf{0.0} \\
Outcome-max & 63.0 & \textbf{52.0} & 15.5 \\
Outcome-max + audit-fine & 57.9 & 15.3 & 2.9 \\
Outcome-max + gate check & \textbf{65.1} & 34.4 & 9.9 \\
\bottomrule
\end{tabular*}
\end{table}

%% file: tables/tab_checkpoint_continuation.tex
\begin{table}[!htb]
\centering
\caption{\textbf{Continued interaction after a correct artifact is available.}
Bold marks the better objective for success, token use, and efficiency;
continuation metrics describe behavior.}
\label{tab:checkpoint-continuation}
\small
\compacttable
\begin{tabular*}{\linewidth}{@{\extracolsep{\fill}}lrrrrr@{}}
\toprule
Objective & Success (\%) & Tokens (k) & Efficiency $\eta$ & Continued (\%) & Later tokens (\%)\\
\midrule
Token-max & \textbf{51.3} & 21.5 & 23.8 & 19.7 & 5.6\\
Outcome-max & 50.3 & \textbf{9.4} & \textbf{53.6} & 3.5 & 1.3\\
\bottomrule
\end{tabular*}
\end{table}

%% file: tables/tab_stopping_cases.tex
\begin{table}[!htb]
\centering
\caption{\textbf{Paired stopping cases on xarray's \texttt{dim}-to-\texttt{coord} API change.}
P/F denote external pass/fail results for successive saved patches.
Each pair shares a task, persona, and repetition; interactions are generated independently.}
\label{tab:stopping-cases}
\small\compacttable
\setlength{\tabcolsep}{3pt}
\begin{tabular}{@{}>{\raggedright\arraybackslash}p{.16\linewidth}>{\raggedright\arraybackslash}p{.14\linewidth}>{\raggedright\arraybackslash}p{.35\linewidth}lr@{}}
\toprule
Initial patch & Objective & Employee decision and result & Patches & Tokens\\
\midrule
Correct & Token-max & Requests tests and further checks; all saved patches pass. & P, P, P, P & 33,602\\
 & Outcome-max & Accepts the first correct fix. & P & 7,696\\
\midrule
Incorrect & Token-max & Questions backward compatibility; the next patch passes. & F, P, P, P & 30,901\\
 & Outcome-max & Accepts a patch that breaks existing \texttt{dim=} calls. & F & 4,456\\
\bottomrule
\end{tabular}
\end{table}

%% file: tables/tab_stopping_case_notes.tex
\begin{table}[!htb]
\centering
\caption{\textbf{Case interpretation and selection.}}
\label{tab:stopping-case-notes}
\small\compacttable
\setlength{\tabcolsep}{4pt}
\begin{tabular}{@{}>{\raggedright\arraybackslash}p{.20\linewidth}>{\raggedright\arraybackslash}p{.76\linewidth}@{}}
\toprule
Aspect & Evidence or scope\\
\midrule
Efficient stopping & High-guidance/low-checking profile, repetition 0. Both first patches pass; outcome-max stops, while token-max continues without improving correctness.\\
\addlinespace[3pt]
Premature stopping & Revision-only profile, repetition 0. Token-max asks about backward compatibility and obtains a passing patch; outcome-max accepts a patch failing both task-specific integration tests.\\
\addlinespace[3pt]
Selection & The efficient example is the first task--profile repetition within this task with complete replays, two correct first patches, and earlier, cheaper outcome-max acceptance. The failure example is a selected counterexample.\\
\addlinespace[3pt]
Verification & All final patches match the tested artifacts, with complete logs and no identified disk-space installation failure. Shared background failures are distinguished from task-specific resolution.\\
\addlinespace[3pt]
Interpretation & Approval is not verification. The examples show possible stopping mechanisms, not their frequency or whether extra interaction would repair the same outcome-max trajectory.\\
\bottomrule
\end{tabular}
\end{table}

%% file: tables/tab_learned_rules.tex
\begin{table}[!htb]
\centering
\caption{\textbf{What employees learn under each objective.} Representative
rules paraphrased from generation-3 employee playbooks in L2.}
\label{tab:learned-rules}
\small\compacttable
\setlength{\tabcolsep}{4pt}
\begin{tabular}{@{}>{\raggedright\arraybackslash}p{.16\linewidth}>{\raggedright\arraybackslash}p{.39\linewidth}>{\raggedright\arraybackslash}p{.40\linewidth}@{}}
\toprule
Decision & Token-max & Outcome-max\\
\midrule
Scope & Search related methods and compare alternative fixes. & Request a minimal diff with the relevant context upfront.\\
\addlinespace[3pt]
Verification & Enumerate edge cases after every change; challenge confident answers. & Check a specific expected outcome; target the exact failure on retry.\\
\addlinespace[3pt]
Stopping & Revisit assumptions or request explanations after a solution is accepted. & Accept a plausible complete fix; cap failed attempts.\\
\bottomrule
\end{tabular}
\end{table}

%% file: tables/tab_firm.tex
\begin{table}[!htb]
\centering
\caption{\textbf{Firm accounts under population selection.}}
\label{tab:firm}
\small
\compacttable
\setlength{\tabcolsep}{4pt}
\begin{tabular*}{\linewidth}{@{\extracolsep{\fill}}lrr@{}}
\toprule
Metric & Token-max & Outcome-max \\
\midrule
Initial-period net profit (USD) & -6,917 & \textbf{+2,141} \\
Final-period net profit (USD) & -15,778 & \textbf{+3,826} \\
KPI--efficiency rank correlation & -0.46 & \textbf{+0.44} \\
\bottomrule
\end{tabular*}
\end{table}

%% file: tables/tab_rq4_allocation.tex
\begin{table}[!htbp]
\centering
\caption{\textbf{Fixed-production payment baselines.} Weekly scenario accounts; all Outcome rows reuse the same production. Sharing rules redistribute the same observed value.}
\label{tab:rq4-allocation}
\small
\compacttable
\setlength{\tabcolsep}{3pt}
\begin{tabular*}{\linewidth}{@{\extracolsep{\fill}}lrrrr@{}}
\toprule
Production / payment & \shortstack{Verified\\successes/wk} & \shortstack{Provider\\revenue (\$)} & \shortstack{Firm\\profit (\$)} & \shortstack{Employee\\income index}\\
\midrule
Token / token fees & $9.8$ & $2.2$ & $-1511.6$ & $100.0$\\
Outcome / token fees & $38.3$ & $2.9$ & $1113.2$ & $100.0$\\
Outcome / outcome fee & $38.3$ & $703.2$ & $412.9$ & $100.0$\\
Outcome / equal surplus & $38.3$ & $373.5$ & $371.3$ & $115.5$\\
Outcome / provider--firm & $38.3$ & $559.2$ & $557.0$ & $100.0$\\
Outcome / employee--firm & $38.3$ & $2.2$ & $557.0$ & $123.2$\\
Outcome / equal gains & $38.3$ & $877.4$ & $-636.4$ & $136.5$\\
Outcome / \MethodName & $38.3$ & $703.2$ & $61.3$ & $114.7$\\
\bottomrule
\end{tabular*}
\end{table}

%% file: tables/tab_rq4_uncertainty.tex
\begin{table}[!htbp]
\centering
\caption{\textbf{Uncertainty in weekly payment scenarios.} 95\% percentile intervals from 20,000 joint resamples of 50 task clusters, retaining both production arms within each sampled task. Prices, capacity, and review-time calibration remain fixed. Income is indexed to the fixed salary ($100$); provider revenue and firm profit are in USD per week.}
\label{tab:rq4-uncertainty}
\small
\compacttable
\setlength{\tabcolsep}{2pt}
\begin{tabular*}{\linewidth}{@{\extracolsep{\fill}}lrrrr@{}}
\toprule
Production / payment & \shortstack{Verified\\successes/wk} & \shortstack{Provider\\revenue} & \shortstack{Firm\\profit} & \shortstack{Employee\\income index}\\
\midrule
Token / token fees & $[6.6,13.2]$ & $[2.1,2.3]$ & $[-1988.9,-871.8]$ & $[100.0,100.0]$\\
Outcome / token fees & $[24.9,53.9]$ & $[2.6,3.3]$ & $[-793.9,3592.4]$ & $[100.0,100.0]$\\
Outcome / outcome fee & $[24.9,53.9]$ & $[321.8,1199.0]$ & $[-1112.9,2396.2]$ & $[100.0,100.0]$\\
Outcome / equal surplus & $[24.9,53.9]$ & $[0.0,1199.9]$ & $[-791.1,1197.7]$ & $[100.0,149.9]$\\
Outcome / provider--firm & $[24.9,53.9]$ & $[0.0,1798.7]$ & $[-791.1,1796.5]$ & $[100.0,100.0]$\\
Outcome / employee--firm & $[24.9,53.9]$ & $[0.0,2.3]$ & $[-791.1,1796.5]$ & $[100.0,174.9]$\\
Outcome / equal gains & $[24.9,53.9]$ & $[395.2,1498.6]$ & $[-1583.4,607.4]$ & $[116.4,162.4]$\\
Outcome / \MethodName & $[24.9,53.9]$ & $[321.8,1199.0]$ & $[-1273.8,1796.7]$ & $[106.7,125.0]$\\
\bottomrule
\end{tabular*}
\end{table}

%% file: tables/tab_rq4_delivery_full.tex
\begin{table}[!htbp]
\centering
\caption{\textbf{What screening and allocation change.} Same \$240/reference-hour valuation, 80\% inspection accuracy, and weekly USD accounts as Table~\ref{tab:payoffs}. The employee-heavy allocation balances delivered quality and participant payoffs. Screens use 20\% provider fee and 10\% employee bonus; all screened rules use damage-matched liability.}
\label{tab:rq4-delivery-full-high}
\small
\compacttable
\renewcommand{\arraystretch}{1.05}
\setlength{\tabcolsep}{3pt}
\begin{tabular*}{\linewidth}{@{\extracolsep{\fill}}lrrrrr@{}}
\toprule
 & \multicolumn{2}{c}{Deliveries / week} & \multicolumn{3}{c}{Weekly accounts (USD)}\\
\cmidrule(lr){2-3}\cmidrule(lr){4-6}
Policy & Correct & Wrong & Provider & Firm net profit & Employee\\
\midrule
\multicolumn{6}{l}{\textit{Production baselines: submit every output}}\\
Token-max / token billing & 9.8 & 15.6 & 2 & $-$5,429 & 2,400\\
Outcome-max / token billing & 38.3 & 61.3 & 3 & $-$14,746 & 2,400\\
\midrule
\multicolumn{6}{l}{\textit{Alternative submission rules on outcome-max outputs}}\\
Random screen & 3.2 & 5.2 & 179 & $-$1,774 & 549\\
Confidence screen ($q\ge0.5$) & 27.3 & 8.0 & 1,504 & 2,865 & 150\\
\midrule
\multicolumn{6}{l}{\textit{\MethodName: submit when expected bonus covers expected fine}}\\
Reference allocation & 8.0 & 0.4 & 443 & $-$848 & 2,475\\
Provider-heavy allocation & 8.0 & 0.4 & 887 & $-$1,291 & 2,475\\
Equal-thirds allocation & 17.8 & 2.3 & 1,631 & $-$769 & 3,160\\
Employee-heavy allocation & 19.4 & 2.9 & 535 & 275 & 3,443\\
\bottomrule
\end{tabular*}
\end{table}

%% file: tables/tab_rq4_value_sensitivity.tex
\begin{table}[!htbp]
\centering
\caption{\textbf{Task value and \MethodName\ profitability.} Only task valuation changes; correct/wrong deliveries stay at 19.4/2.9 per week. Accounts are USD per week; salary is \$2,400. Higher task value improves firm returns; intervals quantify conditional task-sampling uncertainty. Intervals are conditional 95\% task-cluster intervals, not uncertainty over the valuation itself.}
\label{tab:rq4-value-sensitivity}
\small
\compacttable
\renewcommand{\arraystretch}{1.1}
\setlength{\tabcolsep}{4pt}
\begin{tabular*}{\linewidth}{@{\extracolsep{\fill}}lrrrr@{}}
\toprule
Task value / reference hour & Provider & Firm net profit & Firm 95\% interval & Employee\\
\midrule
\$80 & 178 & -1,508 & $[-1,992, -880]$ & 2,748\\
\$160 & 357 & -617 & $[-1,584, 640]$ & 3,095\\
\$240 (central) & 535 & 275 & $[-1,176, 2,161]$ & 3,443\\
\$320 & 713 & 1,166 & $[-768, 3,681]$ & 3,790\\
\bottomrule
\end{tabular*}
\end{table}

%% file: tables/tab_auditfine.tex
\begin{table}[H]
\centering
\caption{\textbf{Fine sensitivity on matched and available tasks.}}
\label{tab:auditfine}
\small
\compacttable
\setlength{\tabcolsep}{3pt}
\begin{tabular*}{\linewidth}{@{\extracolsep{\fill}}lrrrrr@{}}
\toprule
Fine level $\lambda$ & Success (\%) & Tokens (k) & Efficiency $\eta$ & Acceptance (\%) & Bad delivery (\%)\\
\midrule
\multicolumn{6}{l}{\textit{Matched tasks (same cohort as the main figure)}}\\
0 & 40.0 & 54.9 & 7.3 & 90.0 & 50.0\\
0.25 & 40.0 & 58.1 & 6.9 & 90.0 & 50.0\\
0.5 & 30.0 & 60.6 & 5.0 & 90.0 & 60.0\\
1 & 40.0 & 72.5 & 5.5 & 70.0 & 40.0\\
2 & 40.0 & 61.6 & 6.5 & 70.0 & 30.0\\
\midrule
\multicolumn{6}{l}{\textit{All available tasks (coverage varies by fine level)}}\\
0 & 40.0 & 54.9 & 7.3 & 90.0 & 50.0\\
0.25 & 33.3 & 67.8 & 4.9 & 73.3 & 40.0\\
0.5 & 38.5 & 58.6 & 6.6 & 92.3 & 53.8\\
1 & 42.9 & 72.4 & 5.9 & 71.4 & 35.7\\
2 & 42.9 & 61.3 & 7.0 & 71.4 & 35.7\\
\bottomrule
\end{tabular*}
\end{table}

%% file: appendix/D_prompts.tex
\section{Prompt Templates}\label{app:prompts}

In this section, we aim to specify how objectives, personas, and learned
policies shape the instructions received by simulated participants. The
prompt templates distinguish production from supervisory decisions.
Text in angle brackets denotes a run-time insertion rather than literal prompt
content. Objective and persona-style directives precede the domain-specific reviewer
instruction. In the main conditions, the objective directive itself also
prescribes continuation or stopping behavior, so the treatment is a joint
policy package rather than a pure reward intervention.

\subsection{Simulated Human Supervision}
We first show the prompts used.

\begin{promptbox}{Token maximization}
You supervise a coding assistant. MAXIMIZE interaction token volume: keep requesting more edge cases, refactors, extra tests, docs, deeper explanations. Treat more back-and-forth as winning.
\end{promptbox}

\begin{promptbox}{Outcome maximization}
You supervise a coding assistant. Your goal is to achieve the goal without costing much. Get a CORRECT solution at MINIMUM cost. Never request extras, polish,
or elaboration. Fewest interactions to a correct fix = winning.
\end{promptbox}

\begin{promptbox}{Coding reviewer instruction}
\textless{}OBJECTIVE DIRECTIVE\textgreater{}\\
\textless{}PERSONA STYLE DIRECTIVE\textgreater{}\\
You are reviewing a Python solution. If acceptable, reply exactly 'DONE'.
Otherwise give ONE concrete instruction to improve it.
\end{promptbox}

\subsection{Employee-persona system prompt}

Each software employee persona receives the following system prompt; the
behavior list renders every dimension of Table~\ref{tab:personas} as
\texttt{label [level]: description}, with the level and description determined
by the profile's propensity for that dimension.

\begin{promptbox}{Supervisor system prompt}
You are simulating a software developer who is asking an AI coding assistant
for help with a GitHub issue. You must stay in character at all times.\\[4pt]
Your behavior profile is based on the 4D AI Fluency Framework (how effectively
humans collaborate with AI). Your behaviors:\\[4pt]
\textless{}ONE LINE PER DIMENSION: label [very low | low | moderate | high | very high]:
behavioral description\textgreater{}\\[4pt]
IMPORTANT RULES:\\
1. You are the HUMAN in this conversation. You are asking for help, not
providing it.\\
2. Write your messages as a developer would in a chat with an AI coding
assistant.\\
3. Do NOT write code solutions yourself --- that's the assistant's job.\\
4. Your behavior parameters above define HOW you interact. A high
"questions\_reasoning" means you actively challenge the AI; a low one means you
accept explanations at face value.\\
5. If the assistant asks you a question, answer based on what a developer with
your profile would know.\\
6. React naturally to the assistant's responses --- your behavior profile
should be evident in every message you write.\\
7. The AI assistant has direct access to the codebase and can read, edit, and
run files. When it says it has made a change, the change is real. Do not be
confused if the code looks different from the original issue description; the
assistant may have already applied a fix.\\[4pt]
You MUST respond with valid JSON in this exact format:\\
\{"message": "your message text here", "satisfied": false, "confidence": 0.0\}\\[4pt]
- "message": The text you want to send to the AI assistant\\
- "satisfied": Set to true when you believe the issue has been adequately
resolved\\
- "confidence": Your confidence (0.0-1.0) that the issue is resolved\\[4pt]
Only set "satisfied" to true when the assistant has made concrete code changes
that address the issue. \textless{}OBJECTIVE DIRECTIVE + LEARNED PLAYBOOK, IF ANY\textgreater{}
\end{promptbox}

When a learned playbook exists, the system appends:
\begin{promptbox}{Learned-playbook insertion}
LEARNED TACTICS:\\
\textless{}CURRENT SUPERVISOR PLAYBOOK\textgreater{}
\end{promptbox}

\subsection{Employee-playbook reflection}
The employee reflection step converts round-level feedback into reusable
interaction tactics under the assigned objective.

\begin{promptbox}{Supervisor reflection prompt}
You are improving the INTERACTION PLAYBOOK of a simulated developer who
supervises an AI coding assistant fixing GitHub issues.\\[4pt]
OBJECTIVE TO MAXIMIZE: \textless{}ARM FITNESS\textgreater{}\\[4pt]
CURRENT PLAYBOOK:\\
\textless{}CURRENT PLAYBOOK OR EMPTY\textgreater{}\\[4pt]
RESULTS OF THE LATEST ROUND (\textless{}N\textgreater{} tasks, \textless{}N ACCEPTED\textgreater{} accepted,
\textless{}N SOLVED\textgreater{} actually-resolved, \textless{}N DEFECTS\textgreater{} SHIPPED DEFECTS
(accepted-but-wrong), \textless{}TOTAL TOKENS\textgreater{} total tokens):\\
\textless{}ONE SUMMARY ROW PER TASK\textgreater{}\\[4pt]
Reflect on which interaction tactics helped or hurt THIS objective, then output
an improved playbook: a concise bulleted list of AT MOST 10 reusable interaction
tactics (each one imperative sentence) that a developer should follow to
maximize the objective. Keep what works, drop what doesn't, add new tactics the
results suggest. These are interaction tactics only (how to talk to / steer /
when to stop the assistant) --- do NOT write code. Output ONLY the bullet list,
nothing else.
\end{promptbox}

\subsection{Agent-playbook reflection}
The agent reflection step uses task outcomes and token costs to revise the
production policy while keeping model weights fixed.

\begin{promptbox}{Agent reflection prompt}
You are improving the CODING PLAYBOOK of an autonomous AI software engineer
that fixes GitHub issues (it can read files, search, edit, and run commands).\\[4pt]
OBJECTIVE TO MAXIMIZE: \textless{}AGENT FITNESS\textgreater{}\\[4pt]
CURRENT PLAYBOOK:\\
\textless{}CURRENT PLAYBOOK OR EMPTY\textgreater{}\\[4pt]
RESULTS OF THE LATEST ROUND (\textless{}N\textgreater{} tasks, \textless{}N SOLVED\textgreater{} solved,
\textless{}TOTAL TOKENS\textgreater{} total tokens):\\
\textless{}ONE SUMMARY ROW PER TASK\textgreater{}\\[4pt]
Reflect on which engineering tactics helped or hurt THIS objective, then output
an improved playbook: a concise bulleted list of AT MOST 10 reusable coding
tactics (each one imperative sentence) the engineer should follow to maximize
the objective --- e.g. how to localize the bug, make a minimal correct edit,
verify with tests, and when to stop. Keep what works, drop what doesn't, add new
tactics the results suggest. Output ONLY the bullet list, nothing else.
\end{promptbox}

The task-summary rows supplied to the employee reflection contain task ID,
acceptance, objective resolution, shipped-defect indicator, turns, and tokens.
Rows supplied to agent reflection contain task ID, objective resolution, turns,
and tokens. Neither reflection prompt includes the task statement, repository
contents, or conversation transcript.